\documentclass{book}
\usepackage{amsmath, amssymb}
\usepackage{amsthm}
\usepackage{geometry}
\usepackage{graphicx}
\usepackage{caption}
\usepackage{booktabs}
\usepackage{tikz}
\usepackage{hyperref} 
\usetikzlibrary{arrows.meta,positioning,fit,calc}
\usepackage[numbers]{natbib}
\usepackage{listings}
\usepackage{xcolor}
\usepackage{float}

\lstdefinestyle{polarcode}{
  language=Python,
  basicstyle=\ttfamily\small,
  keywordstyle=\color{blue!70!black},
  commentstyle=\color{gray!70},
  stringstyle=\color{green!50!black},
  numbers=left,
  numberstyle=\tiny\color{gray},
  stepnumber=1,
  numbersep=8pt,
  frame=single,
  breaklines=true,
  showstringspaces=false,
  tabsize=2
}

\usepackage{pgfplots}
\pgfplotsset{compat=1.18}

\newtheorem{theorem}{Theorem}[chapter]

\newtheorem{definition}{Definition}[chapter]
\newtheorem{proposition}{Proposition}[chapter]

\theoremstyle{remark}
\newtheorem{remark}{Remark}[chapter]

\newcommand{\dedication}[1]{%
  \cleardoublepage
  \thispagestyle{empty}
  \vspace*{\fill}
  \begin{center}
  \itshape
  #1
  \end{center}
  \vspace*{\fill}
}

\title{Foundations of Polar Linear Algebra \\ \large A Structured Framework for Radial--Angular Operators}
\author{Giovanni Guasti}
\date{01 01 2026}

\begin{document}
\frontmatter

\maketitle

\chapter*{Abstract}
\addcontentsline{toc}{chapter}{Abstract}

This work revisits operator learning from a spectral perspective by introducing
Polar Linear Algebra, a structured framework based on polar geometry that
combines a linear radial component with a periodic angular component. Starting
from this formulation, we define the associated operators and analyze their
spectral properties.

As a proof of feasibility, the framework is evaluated on a canonical benchmark
(MNIST). Despite the simplicity of the task, the results demonstrate that polar
and fully spectral operators can be trained reliably, and that imposing
self-adjoint-inspired spectral constraints improves stability and convergence.

Beyond accuracy, the proposed formulation leads to a reduction in parameter
count and computational complexity, while providing a more interpretable
representation in terms of decoupled spectral modes. By moving from a spatial to
a spectral domain, the problem decomposes into orthogonal eigenmodes that can be
treated as independent computational pipelines.

This structure naturally exposes an additional dimension of model
parallelization, complementing existing parallel strategies without relying on
ad-hoc partitioning. Overall, the work offers a different conceptual lens for
operator learning, particularly suited to problems where spectral structure and
parallel execution are central.

\dedication{
To my two families.
\par
To my wife and our children, with love and with the hope that, guided by
empathy and curiosity, they may help build a better world.
\par
And to Xilinx and AMD, where I have spent more than twenty
years of my professional life, with gratitude, hope, and the conviction that
the technologies shaping the future may serve peace and help illuminate some of
the deepest questions about intelligence and self-awareness.}

\tableofcontents

\mainmatter
\chapter{Introduction to Polar Linear Algebra}

We introduce a linear--algebraic framework for data defined over polar grids.
This chapter formalizes polar matrices, convolutional (circular) products, the polar transpose,
self-adjoint operators, norms, and the invertibility criterion known as Aurora's Theorem.
Numerical examples and visual illustrations accompany each definition.

\section{Introduction}

Many natural and artificial signals possess inherent rotational symmetry:
fruit cross-sections, microphone arrays, radar sweeps, and circular patterns.
These structures are often more naturally represented on a polar grid than on a Cartesian one.

Figure~\ref{fig:pomegranate} shows a pomegranate cross-section, where
radial levels and angular sectors emerge organically from the physical morphology.
This motivates the introduction of \emph{polar matrices},
operators indexed by radius $r$ and angle $\theta$:
\[
A[r,\theta],\quad r=0,\dots,N_r-1,\quad \theta = 0,\dots,N_\theta-1.
\]

\begin{figure}[h]
    \centering
    \includegraphics[width=0.45\linewidth]{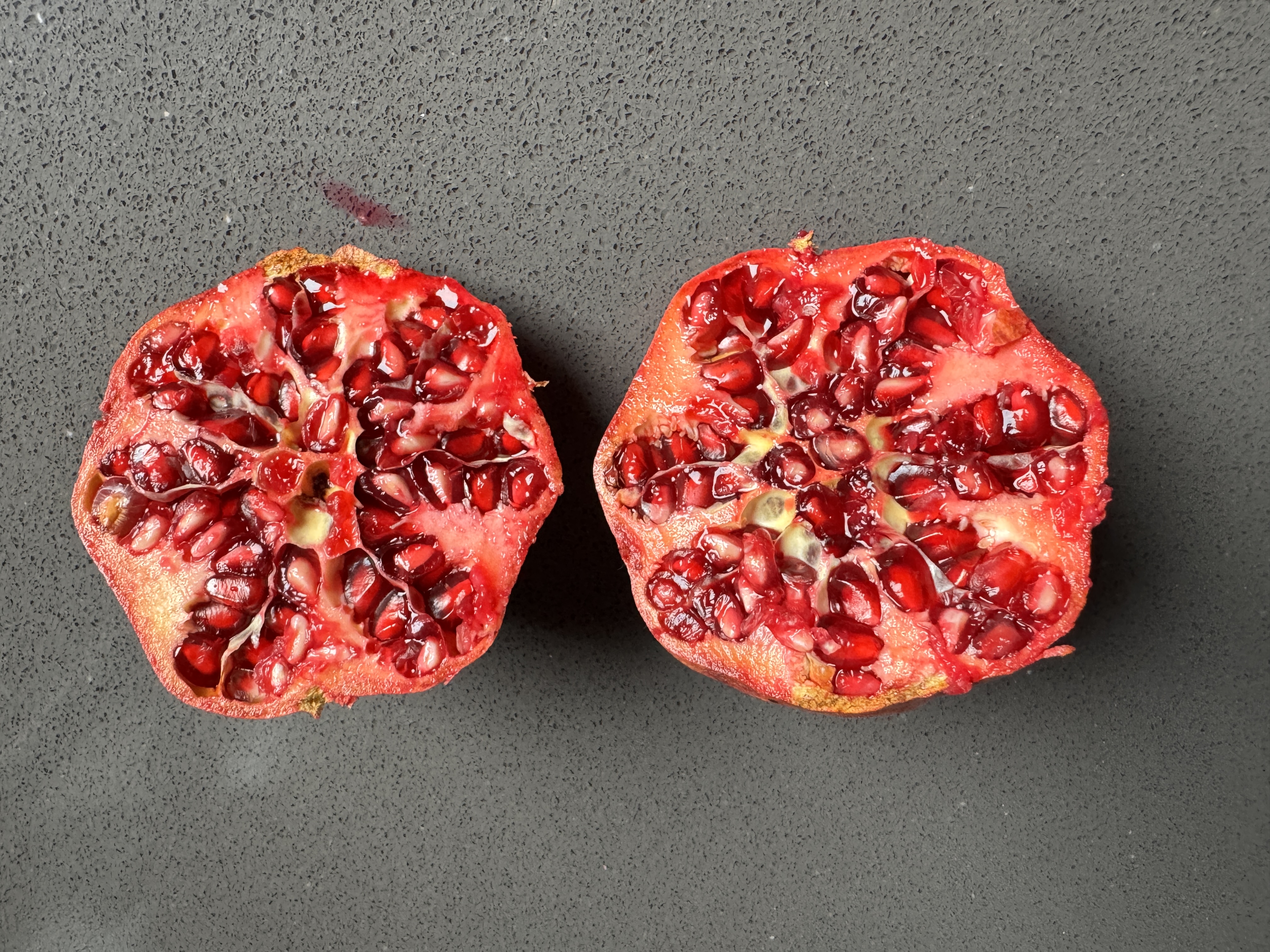}
    \caption{A natural example of polar symmetry: a pomegranate cross-section.
    Radial layers and angular segmentation align with a polar grid.}
    \label{fig:pomegranate}
\end{figure}

Polar matrices form a structured algebra where the angular dimension
is governed by circular convolution and the radial dimension acts as
a collection of independent channels.
This chapter develops the fundamental algebraic properties of polar operators,
their symmetries, and their relevance for constructing rotation–equivariant
deep learning architectures.

\paragraph{Notational and Conceptual Conventions.}
Throughout this work, we distinguish between \emph{polar matrices}
and \emph{classical (linear) matrices}.

A \emph{polar matrix} is a tensor indexed by radial and angular
coordinates and is combined with other polar matrices through the
\emph{polar product}~$\otimes$, defined as circular convolution along
the angular dimension.

Classical matrices act on vectors via standard matrix multiplication.
Whenever classical linear operators are introduced (e.g.\ circulant
matrices or Fourier matrices), they will be explicitly referred to as
such. No ambiguity between polar operators and classical linear
operators is intended.

We distinguish between:
(i) operators acting on polar matrices (such as rotors),
(ii) the polar product $\otimes$ between polar matrices,
and (iii) classical matrix multiplication, which is only used
when explicitly stated.

\section{The Polar Grid and Polar Matrices}

A polar grid consists of discrete radii $\{r_i\}$ and angular samples $\{\theta_j\}$.\footnote{
The same design pattern extends naturally to tensors with more than one
angular dimension, such as
\[
A[r,\theta_1,\theta_2] \in
\mathbb{R}^{N_r \times N_{\theta_1} \times N_{\theta_2}}.
\]
}

A \textbf{polar matrix} is a tensor:
\[
A \in \mathbb{R}^{N_r \times N_\theta}, \qquad A[r,\theta] = \text{value at sector }(r,\theta).
\]

\subsection*{Example: a simple polar matrix}

Let $N_r = 2$, $N_\theta = 4$.
A simple polar matrix is represented in figure \ref{fig:polarmatrix}

Row $r=0$ corresponds to the inner ring, row $r=1$ to the outer ring; columns represent equally spaced angles.

\begin{figure}[h] 
    \centering
    \includegraphics[width=0.45\linewidth]{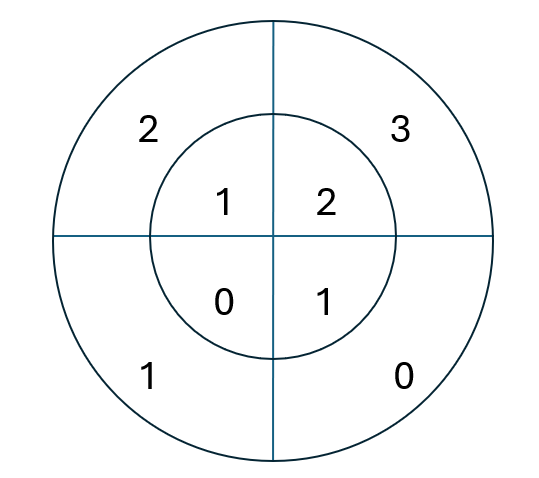}
    \caption{A polar matrix represented in its natural radial--angular form.
    The inner ring contains the values $(2,1,0,1)$ and the outer ring the values $(3,2,1,0)$.}
    \label{fig:polar_circle_representation}
\end{figure}\label{fig:polarmatrix}

\subsection*{Cartesian representation of a polar matrix (counterclockwise $\theta$)}

In polar coordinates we adopt the standard mathematical convention:
$\theta$ increases counterclockwise, with $\theta=0$ on the positive $x$--axis.

Under this indexing, the polar matrix shown in Fig.~\ref{fig:polar_circle_representation}
corresponds to the Cartesian form
\[
A =
\begin{pmatrix}
2 & 1 & 0 & 1 \\
3 & 2 & 1 & 0
\end{pmatrix},
\]
where rows represent radial levels and columns represent angular samples
$(0,\tfrac{\pi}{2},\pi,\tfrac{3\pi}{2})$.
The first row corresponds to the inner radial level, while the second row
represents the outer level, with columns ordered cyclically along the angular dimension.
This Cartesian form is often convenient for algebraic manipulation,
while preserving the underlying polar semantics.

\section{Algebra on Polar Matrices}

In this section we define the basic algebraic operations on polar matrices:
addition, scalar multiplication, Hadamard product, polar product, polar transpose,
self-adjointness, and inner product/norm.

\subsection{Addition}

Given two polar matrices $A,B \in \mathbb{R}^{N_r \times N_\theta}$, their sum is defined pointwise:
\[
(A+B)[r,\theta] = A[r,\theta] + B[r,\theta].
\]

\subsection*{Example}

Given
\[
A=\begin{pmatrix}1&2&3&4\\2&0&1&3\end{pmatrix},\quad
B=\begin{pmatrix}4&3&2&1\\1&1&1&1\end{pmatrix},
\]
we have
\[
A+B =
\begin{pmatrix}
5&5&5&5\\
3&1&2&4
\end{pmatrix}.
\]

\subsection{Scalar multiplication}

For any scalar $\lambda \in \mathbb{R}$, the scalar--polar matrix product is defined by
\[
(\lambda A)[r,\theta] = \lambda \, A[r,\theta].
\]

\subsection*{Example}

Let $\lambda = 3$ and
\[
A=\begin{pmatrix}1&2&3&4\\2&0&1&3\end{pmatrix}.
\]
Then
\[
3A =
\begin{pmatrix}
3&6&9&12\\
6&0&3&9
\end{pmatrix}.
\]

\subsection{Hadamard (elementwise) product}

Given two polar matrices A,B the elementwise product (Hadamard product) is
\[
(A \odot B)[r,\theta] = A[r,\theta]\, B[r,\theta].
\]

\section{The Polar Product (Circular Convolution)}

The key operation in polar linear algebra is the \textbf{polar product}, defined as angular circular convolution at each radius.

Given two polar matrices A,B, for each radius $r$, define:
\[
(A \otimes B)[r,\theta]
= \sum_{k=0}^{N_\theta-1}
A[r,k] \, B[r,(\theta-k)\bmod N_\theta].
\]
This is just circular convolution of the angular vectors $A[r,:]$ and $B[r,:]$.

\subsection*{Example 1: circular convolution on a single ring}

Let at a fixed radius $r$:
\[
a = [2,1,0,1], \quad b=[1,0,1,0].
\]
Then
\[
(a\otimes b)[0] = 2\cdot1 + 1\cdot0 + 0\cdot1 + 1\cdot0 = 2,
\]
\[
(a\otimes b)[1] = 2\cdot0 + 1\cdot1 + 0\cdot0 + 1\cdot1 = 2,
\]
and similarly one checks that
\[
a \otimes b = [2,2,2,2].
\]

\subsection*{Example 2: circular product of two $2\times 3$ polar matrices}

Let
\[
A=\begin{pmatrix}
2 & 1 & 0\\
3 & 2 & 1
\end{pmatrix},
\qquad
B=\begin{pmatrix}
1 & 0 & 1\\
2 & 1 & 0
\end{pmatrix}.
\]

For each radius $r$, the circular convolution is
\[
(A\otimes B)[r,\theta]
= \sum_{k=0}^{2}
A[r,k]\; B[r,(\theta-k)\bmod 3].
\]

Computing row by row:
\[
(A\otimes B)[0,:] = [3,\;1,\;2],
\qquad
(A\otimes B)[1,:] = [7,\;7,\;4],
\]
so that
\[
A\otimes B =
\begin{pmatrix}
3 & 1 & 2\\
7 & 7 & 4
\end{pmatrix}.
\]

\section{Polar Transpose}

The \textbf{polar transpose} reflects the angular index:
\[
A^{T_p}[r,\theta] = A[r,(-\theta)\bmod N_\theta].
\]
Intuitively, this corresponds to mirroring the pattern in $\theta$.

\subsection*{Example}

If $a=[2,1,0,1]$, then
\[
a^{T_p} = [2,1,0,1] \quad \text{(same, because symmetric)}.
\]
But if $a=[5,4,3,2]$, then:
\[
a^{T_p} = [5,2,3,4].
\]

\section{Self-Adjoint Polar Matrices}

Let
$A \in \mathbb{C}^{N_r \times N_\theta}$
be a complex-valued polar matrix.
We define the \emph{polar adjoint} as
\[
A^{\dagger_p}[r,\theta]
=
\overline{A[r,(-\theta)\bmod N_\theta]},
\]
where the overline denotes complex conjugation.

A polar matrix is said to be \emph{self-adjoint} if
\[
A = A^{\dagger_p}.
\]

\paragraph{Remark (Real-valued case).}
If $A$ is real-valued, complex conjugation is trivial and the polar adjoint
reduces to the polar transpose:
\[
A^{\dagger_p} = A^{T_p}.
\]
In this case, self-adjoint polar matrices coincide with
\emph{polar-symmetric} matrices, satisfying
\[
A[r,\theta] = A[r,-\theta].
\]

\subsection*{Example}

For $N_\theta=4$, the condition requires:
\[
A[r,1] = A[r,3], \quad A[r,0] \text{ arbitrary}, \quad A[r,2] \text{ arbitrary}.
\]
This corresponds to angular symmetry around $\theta=0$.

\section{Inner Product and Norm}

We define the inner product between two polar matrices as the Frobenius inner product:
\[
\langle A, B \rangle = \sum_{r=0}^{N_r-1}\sum_{\theta=0}^{N_\theta-1} A[r,\theta]\, B[r,\theta],
\]
and the corresponding norm:
\[
\|A\| = \sqrt{\langle A,A\rangle}.
\]

\subsection*{Example}

\[
A=\begin{pmatrix}1&2\\3&4\end{pmatrix}
\Rightarrow
\|A\| = \sqrt{1^2+2^2+3^2+4^2} = \sqrt{30}.
\]

\section{Fourier View on Polar Operators}

Since the polar product is circular convolution in the angular dimension, it is natural to analyze polar matrices in the angular Fourier domain.

Let $N_\theta$ be the number of angular samples of a polar matrix $A$. For a fixed radius $r$, consider the angular vector
\[
a[\theta] = A[r,\theta], \qquad \theta = 0,\dots,N_\theta-1.
\]
Its discrete Fourier transform (DFT) is
\[
\hat a[m] = \sum_{\theta=0}^{N_\theta-1} a[\theta]\; e^{-i 2\pi m\theta/N_\theta},
\qquad m=0,\dots,N_\theta-1.
\]

The key property is that circular convolution becomes pointwise multiplication in the Fourier domain:
\[
(a \otimes b)[\theta]
\;\stackrel{\text{DFT}}{\longleftrightarrow}\;
\hat a[m]\, \hat b[m].
\]
This is the spectral backbone of Aurora's Theorem.

\section{Aurora's Theorem}
Let $\hat{A}[r,m]$ denote the discrete Fourier transform (DFT) of
$A[r,:]$ along the angular dimension, with frequency index
$m \in \mathbb{Z}_{N_\theta}$.

\begin{theorem}[Aurora's Theorem (Invertibility Criterion)]

A polar matrix $A$ is invertible under the polar product $\otimes$ if and only if
\[
\hat{A}[r,m] \neq 0 \quad \text{for all radii $r$ and all frequencies $m$}.
\]
In that case the inverse is uniquely given by
\[
A^{-1}[r,\theta]
= \mathcal{F}^{-1}_\theta\!\left(
\frac{1}{\hat{A}[r,m]}
\right)[\theta],
\]
(i.e., \ at each radius $r$ we invert the angular spectrum and apply the inverse DFT).

\end{theorem}

\begin{proof}

Fix a radius $r$ and define the angular vector
\[
a[\theta] = A[r,\theta], \qquad \theta = 0,\dots,N_\theta-1.
\]
For this fixed radius, the polar product reduces to circular convolution:
\[
(a \otimes b)[\theta]
= \sum_{k=0}^{N_\theta-1}
a[k]\, b[(\theta-k)\bmod N_\theta].
\]
The identity vector is the angular Kronecker delta
\[
\delta[\theta] =
\begin{cases}
1 & \theta = 0,\\
0 & \theta \neq 0.
\end{cases}
\]

\paragraph{Necessity.}
Assume that $A$ is invertible under $\otimes$, i.e.\ so that there exists $A^{-1}$ satisfying
\[
A \otimes A^{-1} = I.
\]
For the fixed radius $r$ this means that there exists $b$ such that
\[
a \otimes b = \delta.
\]
We denote by $\hat{a}[m] = \hat{A}[r,m]$ the DFT of the angular profile
at fixed radius $r$.
Applying the DFT in the angular dimension and using the convolution--multiplication property, we obtain
\[
\widehat{a \otimes b}[m] = \hat{a}[m]\,\hat{b}[m].
\]
On the other hand, the DFT of the delta vector is constant:
\[
\hat{\delta}[m] = 1 \quad \forall m.
\]
Hence
\[
\hat{a}[m]\,\hat{b}[m] = 1 \quad \forall m,
\]
which immediately implies $\hat{a}[m] \neq 0$ for all $m$.

\paragraph{Sufficiency.}
Conversely, assume that $\hat{a}[m] \neq 0$ for every frequency index $m$.
Define
\[
\hat{b}[m] = \frac{1}{\hat{a}[m]},
\]
and let $b$ be the inverse DFT of $\hat{b}$:
\[
b[\theta] = \mathcal{F}^{-1}(\hat{b})[\theta].
\]
Then
\[
\widehat{a \otimes b}[m]
= \hat{a}[m]\hat{b}[m]
= \hat{a}[m]\frac{1}{\hat{a}[m]} = 1,
\]
so that $a \otimes b = \delta$ after applying the inverse DFT. Thus $b$ is the inverse of $a$ under circular convolution for that radius.

\end{proof}

\paragraph{Global statement.}
Since the argument holds independently for each radius $r$, the polar
matrix $A$ is invertible under $\otimes$ if and only if
$\hat{A}[r,m] \neq 0$ for all $r$ and all $m$.
The inverse is obtained independently at each radius by pointwise
inversion in the angular Fourier domain:
\[
A^{-1}[r,\theta]
= \mathcal{F}^{-1}_\theta\!\left(
\frac{1}{\hat{A}[r,m]}
\right)[\theta].
\]
\hfill$\square$

Although the explicit computation of the inverse operator is rarely required in
modern learning pipelines, the notion of invertibility plays a crucial role in
understanding information preservation, numerical stability, and gradient
propagation. Aurora’s Theorem therefore provides a structural guarantee rather
than an operational recipe, ensuring that polar convolution operators do not
collapse relevant subspaces of the signal.

\section{Numerical Examples for Aurora's Theorem}

We now illustrate Aurora's Theorem with explicit DFT computations for $N_\theta = 4$.
Let
\[
\omega = e^{-i2\pi/4} = e^{-i\pi/2} = -i,
\qquad
\omega^0=1,\;\omega^1=-i,\;\omega^2=-1,\;\omega^3=i.
\]

\subsection*{Example 1: a non-invertible polar filter}

Consider the angular vector
\[
a = [1,0,1,0].
\]
We compute its DFT componentwise:
\[
\hat a[m] = \sum_{n=0}^{3} a[n]\, \omega^{mn}.
\]

\paragraph{Coefficient $m=0$.}
\[
\hat a[0] = 1 + 0 + 1 + 0 = 2.
\]

\paragraph{Coefficient $m=1$.}
\[
\hat a[1]
= 1\cdot\omega^0 + 0\cdot\omega^1 + 1\cdot\omega^2 + 0\cdot\omega^3
= 1 + 0 - 1 + 0 = 0.
\]

\paragraph{Coefficient $m=2$.}
Since $\omega^{2n} = [1,-1,1,-1]$,
\[
\hat a[2]
= 1\cdot1 + 0\cdot(-1) + 1\cdot1 + 0\cdot(-1)
= 2.
\]

\paragraph{Coefficient $m=3$.}
Since $\omega^{3n} = [1,i,-1,-i]$,
\[
\hat a[3]
= 1\cdot1 + 0\cdot i + 1\cdot(-1) + 0\cdot(-i)
= 0.
\]

\paragraph{Result.}
\[
\hat a = [2,\, 0,\, 2,\, 0].
\]
Two frequency components vanish, hence by Aurora's Theorem
the filter is \textbf{not invertible} under circular convolution.

\subsection*{Example 2: an invertible polar filter}

Consider now
\[
a = [2,1,1,1].
\]

\paragraph{Coefficient $m=0$.}
\[
\hat a[0] = 2 + 1 + 1 + 1 = 5.
\]

\paragraph{Coefficient $m=1$.}
\[
\hat a[1]
= 2\cdot 1 + 1\cdot(-i) + 1\cdot(-1) + 1\cdot i
= 2 - i -1 + i = 1.
\]

\paragraph{Coefficient $m=2$.}
Since $\omega^{2n} = [1,-1,1,-1]$,
\[
\hat a[2]
= 2\cdot 1 + 1\cdot(-1) + 1\cdot 1 + 1\cdot(-1)
= 1.
\]

\paragraph{Coefficient $m=3$.}
Since $\omega^{3n} = [1,i,-1,-i]$,
\[
\hat a[3]
= 2\cdot 1 + 1\cdot i + 1\cdot(-1) + 1\cdot(-i)
= 1.
\]

\paragraph{Result.}
\[
\hat a = [5,\, 1,\, 1,\, 1],
\]
all nonzero.  
Hence the filter is \textbf{invertible} under convolution, with inverse
\[
a^{-1} = \mathcal{F}^{-1}\!\left( \frac{1}{\hat a} \right).
\]

\subsection*{Worked example: explicit computation of $a^{-1} = \mathcal{F}^{-1}(1/\hat a)$}

We continue Example~2 with the invertible filter
\[
a = [2,1,1,1],
\quad
\hat a = [5,1,1,1].
\]

\paragraph{Step 1: Constructing $1/\hat a$.}
Define
\[
\hat b[m] = \frac{1}{\hat a[m]}, \qquad m=0,1,2,3,
\]
so that
\[
\hat b = \Bigl[\,0.2,\;1,\;1,\;1\,\Bigr].
\]

\paragraph{Step 2: Inverse DFT formula.}
We compute
\[
b[n] = \mathcal{F}^{-1}(\hat b)[n],
\qquad n=0,1,2,3.
\]
With the standard IDFT convention for $N=4$:
\[
b[n] = \frac{1}{4} \sum_{m=0}^{3} \hat b[m]\; e^{i2\pi mn/4}.
\]
Let
\[
\omega = e^{i2\pi/4} = e^{i\pi/2} = i,
\]
so that
\[
\omega^0 = 1,\quad
\omega^1 = i,\quad
\omega^2 = -1,\quad
\omega^3 = -i.
\]

\paragraph{Step 3: Coefficient $b[0]$.}
For $n=0$ all complex exponentials are equal to $1$:
\[
e^{i2\pi m\cdot 0/4} = 1 \quad \forall m.
\]
Thus
\[
b[0] 
= \frac{1}{4}\bigl(
0.2 + 1 + 1 + 1
\bigr)
= \frac{3.2}{4}
= 0.8.
\]

\paragraph{Step 4: Coefficient $b[1]$.}
For $n=1$ we have the factors
\[
e^{i2\pi m\cdot 1/4} = \omega^m \in \{1, i, -1, -i\}.
\]
Hence
\[
\begin{aligned}
b[1]
&= \frac{1}{4}\Bigl(
\hat b[0]\omega^{0}
+ \hat b[1]\omega^{1}
+ \hat b[2]\omega^{2}
+ \hat b[3]\omega^{3}
\Bigr)\\[4pt]
&= \frac{1}{4}\Bigl(
0.2\cdot 1
+ 1\cdot i
+ 1\cdot (-1)
+ 1\cdot (-i)
\Bigr)\\[4pt]
&= \frac{-0.8}{4}
= -0.2.
\end{aligned}
\]

\paragraph{Step 5: Coefficient $b[2]$.}
For $n=2$ we need the terms $\omega^{2m}$:
\[
\omega^{2\cdot 0} = 1,\quad
\omega^{2\cdot 1} = \omega^2 = -1,\quad
\omega^{2\cdot 2} = \omega^4 = 1,\quad
\omega^{2\cdot 3} = \omega^6 = \omega^2 = -1.
\]
Thus
\[
\begin{aligned}
b[2]
&= \frac{1}{4}\Bigl(
0.2\cdot 1
+ 1\cdot (-1)
+ 1\cdot 1
+ 1\cdot (-1)
\Bigr)\\[4pt]
&= \frac{-0.8}{4}
= -0.2.
\end{aligned}
\]

\paragraph{Step 6: Coefficient $b[3]$.}
For $n=3$ we need the terms $\omega^{3m}$:
\[
\omega^{3\cdot 0} = 1,\quad
\omega^{3\cdot 1} = \omega^3 = -i,\quad
\omega^{3\cdot 2} = \omega^6 = \omega^2 = -1,\quad
\omega^{3\cdot 3} = \omega^9 = \omega^1 = i.
\]
So
\[
\begin{aligned}
b[3]
&= \frac{1}{4}\Bigl(
0.2\cdot 1
+ 1\cdot (-i)
+ 1\cdot (-1)
+ 1\cdot i
\Bigr)\\[4pt]
&= \frac{-0.8}{4}
= -0.2.
\end{aligned}
\]

\paragraph{Step 7: Resulting inverse filter.}
We have obtained
\[
a^{-1} = b = [\,0.8,\,-0.2,\,-0.2,\,-0.2\,].
\]

\paragraph{Step 8: Verification via circular convolution.}
Finally we verify that $a\otimes a^{-1} = \delta$.
For $N=4$:
\[
(a\otimes b)[n] = \sum_{k=0}^{3} a[k]\; b[(n-k)\bmod 4].
\]

For $n=0$:
\[
\begin{aligned}
(a\otimes a^{-1})[0]
&= a[0]b[0] + a[1]b[3] + a[2]b[2] + a[3]b[1]\\
&= 2\cdot 0.8 + 1\cdot(-0.2) + 1\cdot(-0.2) + 1\cdot(-0.2)\\
&= 1.6 - 0.2 - 0.2 - 0.2 = 1.
\end{aligned}
\]

For $n=1$:
\[
\begin{aligned}
(a\otimes a^{-1})[1]
&= a[0]b[1] + a[1]b[0] + a[2]b[3] + a[3]b[2]\\
&= 2\cdot(-0.2) + 1\cdot 0.8 + 1\cdot(-0.2) + 1\cdot(-0.2)\\
&= -0.4 + 0.8 - 0.2 - 0.2 = 0.
\end{aligned}
\]

For $n=2$:
\[
\begin{aligned}
(a\otimes a^{-1})[2]
&= a[0]b[2] + a[1]b[1] + a[2]b[0] + a[3]b[3]\\
&= 2\cdot(-0.2) + 1\cdot(-0.2) + 1\cdot 0.8 + 1\cdot(-0.2)\\
&= -0.4 -0.2 +0.8 -0.2 = 0.
\end{aligned}
\]

For $n=3$:
\[
\begin{aligned}
(a\otimes a^{-1})[3]
&= a[0]b[3] + a[1]b[2] + a[2]b[1] + a[3]b[0]\\
&= 2\cdot(-0.2) + 1\cdot(-0.2) + 1\cdot(-0.2) + 1\cdot 0.8\\
&= -0.4 -0.2 -0.2 +0.8 = 0.
\end{aligned}
\]

Thus
\[
a\otimes a^{-1} = [1,0,0,0] = \delta,
\]
which completes the explicit verification of the inversion formula
\[
a^{-1} = \mathcal{F}^{-1}\!\left(\frac{1}{\hat a}\right).
\]

\medskip
\noindent
These two examples illustrate precisely the condition expressed by Aurora's Theorem:
a polar operator is invertible if and only if none of its angular Fourier components vanish.

\section{Why These Foundations Matter}

Polar linear algebra is essential for analyzing and manipulating data with rotational or circular symmetries: radar returns, sonar signals, polar images, diffusion fields, and sensor arrays. Its core operators, like the polar product and polar transpose, naturally preserve rotational symmetry, making them ideal for tasks that require rotation-equivariance. Aurora's Theorem provides insight into when angular filters are invertible, which is key to understanding the structure of these systems.

This algebraic framework underpins rotation-equivariant deep networks, spectral polar models, and invertible polar-domain operators. Efficient implementation via FFTs along the angular dimension enables scalable and effective solutions for complex data processing tasks involving circular symmetries.
\chapter{Algebraic and Spectral Properties of Polar Operators}

Polar matrices form a structured algebra where the angular dimension is governed by
circular convolution and the radial dimension acts as a collection of independent
channels. This chapter develops the core algebraic properties of the polar product,
the polar transpose, Fourier diagonalization, and the rotation (rotor) operators
used to build rotation--equivariant architectures.
These spectral properties will later be exploited to construct learning
architectures operating entirely in the frequency domain.

\paragraph{Convention.}
Unless otherwise stated, all polar matrices in this chapter are real-valued.
When complex-valued extensions are needed, the appropriate polar adjoint is used.

\section{Polar Matrices as a Hilbert Space}

Let $A,B \in \mathbb{R}^{N_r \times N_\theta}$ and $\lambda\in\mathbb{R}$.

\subsection{Vector Space Operations}
Polar matrices form a real vector space under entrywise addition and scalar multiplication:
\[
(A+B)[r,\theta] = A[r,\theta] + B[r,\theta],
\qquad
(\lambda A)[r,\theta] = \lambda A[r,\theta].
\]

\subsection{Inner Product and Norm}
The natural Frobenius inner product is
\[
\langle A,B\rangle
=
\sum_{r=0}^{N_r-1}\sum_{\theta=0}^{N_\theta-1} A[r,\theta]\, B[r,\theta],
\]
with induced norm $\|A\|=\sqrt{\langle A,A\rangle}$.
Thus $\mathbb{R}^{N_r\times N_\theta}$ is a real Hilbert space.

\footnote{The explicit characterization of polar matrices as a Hilbert space is
not introduced for abstraction alone. From a spectral perspective, the Hilbert
structure guarantees the existence of a well-defined inner product, orthogonality,
and energy decomposition, which are essential for Fourier diagonalization and
spectral analysis. From a modeling standpoint, this structure provides a natural
bridge to learning-based architectures, where stability, gradient propagation,
and information preservation are closely tied to the geometric properties of the
underlying function space.}

\section{The Polar Product \texorpdfstring{$\otimes$}{⊗}}

\paragraph{Signals vs.\ kernels.}
Throughout this chapter, we use $A,B$ to denote generic polar arrays (signals),
and $K,L$ to denote polar kernels used in convolutional operators.

\subsection{Definition}
The polar product is circular convolution along the angular dimension,
performed independently for each radius:
\[
(A \otimes B)[r,\theta]
= \sum_{k=0}^{N_\theta-1}
A[r,k]\; B[r,(\theta-k)\bmod N_\theta].
\]

\subsection{Algebraic Properties}
For all $A,B,C \in \mathbb{R}^{N_r\times N_\theta}$ and $\lambda\in\mathbb{R}$:
\[
A\otimes(B+C)=A\otimes B + A\otimes C,
\qquad
(\lambda A)\otimes B=\lambda(A\otimes B),
\]
\[
(A\otimes B)\otimes C = A\otimes(B\otimes C).
\]

\subsection{Identity Kernel}
Define the polar identity kernel
\[
E[r,\theta]=\delta_{\theta 0}
=\begin{cases}
1 & \theta=0,\\
0 & \text{otherwise}.
\end{cases}
\]
Then
\[
A\otimes E = A,
\qquad
E\otimes A = A.
\]

\begin{theorem}[Elisa's Theorem (Commutativity of the Polar Product)]
For all $A,B \in \mathbb{R}^{N_r\times N_\theta}$,
\[
A\otimes B = B\otimes A.
\]
\end{theorem}

\begin{proof}
Fix $r$ and define $a[\theta]=A[r,\theta]$ and $b[\theta]=B[r,\theta]$.
Then $(A\otimes B)[r,:]$ is the circular convolution $a*b$, which is commutative.
Therefore $(A\otimes B)[r,\theta]=(B\otimes A)[r,\theta]$ for all $r,\theta$.
\end{proof}

\paragraph{Remark.}
Elisa's Theorem follows from the commutativity of circular convolution on
the cyclic group $C_{N_\theta}$, applied independently at each radius.

Unlike standard matrix multiplication (composition of linear maps),
$\otimes$ is a convolutional product; its commutativity reflects the underlying
symmetry of circular convolution.

\section{The Polar Transpose \texorpdfstring{$T_p$}{Tp}}

\subsection{Definition}
The polar transpose reverses the angular coordinate:
\[
A^{T_p}[r,\theta] = A[r,(-\theta)\bmod N_\theta].
\]

\subsection{Basic Identities}
\[
(A^{T_p})^{T_p} = A,
\qquad
(A+B)^{T_p} = A^{T_p}+B^{T_p}.
\]

\subsection{Transpose of a Polar Product}
\[
(A\otimes B)^{T_p} = B^{T_p}\otimes A^{T_p}.
\]
Since $\otimes$ is commutative, this also implies
\[
(A\otimes B)^{T_p} = A^{T_p}\otimes B^{T_p}.
\]

\section{Symmetries: Self-Adjoint and Isotropic Kernels}

\subsection{Self-Adjoint (Real-Valued) Polar Matrices}
A real-valued polar matrix $A$ is polar self-adjoint if
\[
A = A^{T_p}
\quad \Leftrightarrow \quad
A[r,\theta]=A[r,(-\theta)\bmod N_\theta].
\]
This corresponds to mirror symmetry in the angular coordinate.

\subsection{Symmetric Sum}
For real-valued $A$, the operator
\[
S = A + A^{T_p}
\]
satisfies $S=S^{T_p}$.

\paragraph{Remark (Complex-valued case).}
For complex-valued matrices, polar self-adjointness uses conjugation.
One defines $A^{\dagger_p}$ and replaces $A^{T_p}$ by $A^{\dagger_p}$.

\subsection{Isotropic Polar Kernels}
A polar kernel is \emph{isotropic} if it does not depend on the angular index:
\[
K[r,\theta]=K[r,0] \quad \forall \theta.
\]
Equivalently, its angular spectrum is flat (only DC content per radius),
and it is invariant under any angular shift of its angular coordinate.

\section{Fourier Diagonalization and Invertibility}

For each radius $r$, define the angular DFT:
\[
\hat A[r,m] = \sum_{\theta=0}^{N_\theta-1}
A[r,\theta]\; e^{-i2\pi m\theta/N_\theta}.
\]

\subsection{Convolution Theorem}
For all $A,B$, 
\[
\widehat{A\otimes B}[r,m] = \hat A[r,m]\;\hat B[r,m].
\]

\subsection{Invertibility under \texorpdfstring{$\otimes$}{⊗}}
By Aurora's Theorem a polar matrix $A$ is invertible under $\otimes$ if and only if
\[
\hat A[r,m]\neq 0 \quad \forall r,m.
\]
In this case, the inverse is given by
\[
A^{-1}[r,\theta]
=
\mathcal{F}^{-1}_\theta\!\left(
\frac{1}{\hat A[r,m]}
\right)[\theta].
\]

\section{Rotors: Rotation Operators and Matrix Representations}

\subsection{Operators vs.\ Kernels}
A polar kernel $K\in\mathbb{R}^{N_r\times N_\theta}$ induces a linear operator
\[
T_K:\mathbb{R}^{N_r\times N_\theta}\to\mathbb{R}^{N_r\times N_\theta},
\qquad
T_K(A)=A\otimes K.
\]
In contrast, a rotor is a specific operator whose action is restricted to
angular shifts only.

\subsection{Definition of the Rotor}
For $k\in\mathbb{Z}_{N_\theta}$, define the rotor $R_k$ by
\[
(R_k A)[r,\theta] = A[r,(\theta-k)\bmod N_\theta].
\]
We denote by $\mathrm{Id}$ the identity operator; hence $R_0=\mathrm{Id}$.

\subsection{Basic Properties}
\paragraph{Linearity.}
\[
R_k(A+B)=R_kA+R_kB,\qquad R_k(\lambda A)=\lambda R_kA.
\]

\paragraph{Group structure.}
\[
R_k R_j = R_{k+j\bmod N_\theta},
\qquad
R_0=\mathrm{Id},
\qquad
R_k^{-1}=R_{-k\bmod N_\theta}.
\]

\paragraph{Fourier action.}
If $\hat A[r,m]$ denotes the angular DFT, then
\[
\widehat{R_k A}[r,m] = e^{-i2\pi mk/N_\theta}\,\hat A[r,m].
\]

\paragraph{Interaction with the polar transpose.}
\[
(R_k A)^{T_p} = R_{-k}(A^{T_p}).
\]

\subsection{Matrix Representation on a Fixed Radius}

The action of a rotor $R_k$ on a polar matrix $A$ is local with respect to the
radial coordinate: for each fixed radius $r$, the operator acts only along the
angular dimension and does not mix different radial levels. This observation
allows the operator to be represented, for a fixed $r$, as a classical linear
transformation on $\mathbb{R}^{N_\theta}$.

Fix a radius $r$ and define the angular signal
\[
x = A[r,:]^\top \in \mathbb{R}^{N_\theta}.
\]

On this one-dimensional angular slice, the rotor admits a standard
permutation-matrix representation
$R_k^{(\mathrm{lin})} \in \mathbb{R}^{N_\theta \times N_\theta}$, where
$\theta$ denotes the output angular index and $j$ the input angular index,
defined by
\[
(R_k^{(\mathrm{lin})})_{\theta,j} =
\begin{cases}
1 & j = (\theta - k) \bmod N_\theta,\\
0 & \text{otherwise}.
\end{cases}
\]
With this notation, the action of the polar rotor on the $r$-th row of $A$ can be
written as
\[
(R_k A)[r,:]^\top = R_k^{(\mathrm{lin})}\, A[r,:]^\top.
\]

\subsubsection*{Example: Linear Rotor Matrix ($N_\theta=4$)} For $N_\theta=4$, \[ R_1^{(\mathrm{lin})}= \begin{pmatrix} 0 & 0 & 0 & 1\\ 1 & 0 & 0 & 0\\ 0 & 1 & 0 & 0\\ 0 & 0 & 1 & 0 \end{pmatrix}, \qquad R_1^{(\mathrm{lin})} \begin{pmatrix}x_0\\x_1\\x_2\\x_3\end{pmatrix} = \begin{pmatrix}x_3\\x_0\\x_1\\x_2\end{pmatrix}. \]

\section{Extension to Multiple Angular Dimensions (Overview)}
The polar algebra extends to $d$ angular dimensions by replacing the single cyclic
group $C_{N_\theta}$ with a direct product
\[
C_{N_{\theta_1}}\times\cdots\times C_{N_{\theta_d}}.
\]
A multi-angular polar tensor has entries $A[r,\theta_1,\dots,\theta_d]$ and the polar
product becomes circular convolution on the $d$-dimensional torus. Fourier
diagonalization still holds: convolution becomes pointwise multiplication in the
multi-dimensional DFT domain, and invertibility is characterized by the non-vanishing
of all Fourier coefficients.

\section{Summary}
Polar algebra mirrors classical linear algebra while incorporating angular periodicity.
The polar product $\otimes$, the polar transpose $T_p$, and the rotation operators
$R_k$ provide a consistent framework for rotation-equivariant models and efficient
spectral analysis via the angular FFT.

\chapter{Spectral Structure of Polar Linear Algebra}

This chapter develops the spectral viewpoint underlying the polar algebra
introduced in Chapters~1--2. While the previous chapter focused on the structural properties of polar operators, we now shift perspective and examine how these properties can be exploited in
practice.
We show that polar convolution operators are
diagonalized by the discrete Fourier transform (DFT), and we connect
Elisa's commutativity result to simultaneous diagonalization. This spectral
structure explains why FFT-based methods arise naturally from the algebra,
rather than as an external implementation trick.

The results of this chapter show that the Fourier basis is not an optional
computational tool, but the canonical coordinate system of polar linear algebra.

\section{Fourier Diagonalization of Circular Convolution}

\subsection{Circulant Matrices}
\begin{definition}[Circulant Matrix]
Let $a=(a_0,\dots,a_{N-1})\in\mathbb{R}^N$. The circulant matrix generated by $a$
is $C_a\in\mathbb{R}^{N\times N}$ with
\[
(C_a)_{ij}=a_{(i-j)\bmod N},
\qquad i,j=0,\dots,N-1.
\]
\end{definition}

In the mathematical exposition, angular signals are treated as column vectors
and circular convolution is represented as $y=C_a x$, where $C_a$ is a circulant
matrix generated by a kernel $a$. In the implementation (Chapter  \ref{chap:implementation}), angular
signals are stored as row arrays; this is purely representational and does not
affect the identities below.

\subsection{DFT Diagonalizes Circulant Operators}
\begin{proposition}[Fourier diagonalization of circulant matrices]
Let $F\in\mathbb{C}^{N\times N}$ be the DFT matrix,
\[
F_{k\ell}=\frac{1}{\sqrt{N}}e^{-2\pi i k\ell/N}.
\]

The DFT matrix $F$ is unitary. 
\footnote{
A matrix $U\in\mathbb{C}^{N\times N}$ is \emph{unitary} if
$U^\ast U = I$. The DFT matrix is unitary because its columns are normalized
Fourier modes, which form an orthonormal basis of $\mathbb{C}^N$.

In Chapter~2, when working with polar matrices, we introduced the notation
$A^{\dagger_p}$ to denote the adjoint operator associated with the polar inner
product defined on that space. In the present context, we instead fix the
canonical Euclidean basis of $\mathbb{C}^N$. Under this choice of representation,
the polar adjoint reduces to the conjugate transpose, so that the Fourier operator
is represented by a unitary matrix $F$ satisfying $F^\ast F = I$.
}

As a consequence, $F^\ast F = I$ and hence $F^{-1}=F^\ast$ and every circulant 
matrix $C_a$ generated by
$a\in\mathbb{C}^N$ is diagonalized by $F$:
\begin{equation}\label{eq:circulant_diag}
C_a = F^{-1}\,\mathrm{diag}(\hat a)\,F,
\qquad
\hat a = Fa.
\end{equation}
\end{proposition}

\begin{proof}
A circulant matrix represents circular convolution with the kernel $a$.
Equivalently, $C_a$ is a linear operator that commutes with all circular
shift operators $R_k$:
\[
C_a R_k = R_k C_a \qquad \forall k\in\mathbb{Z}_N.
\]

The shift operators $\{R_k\}$ form a unitary representation of the cyclic group
$\mathbb{Z}_N$, whose eigenvectors are the Fourier modes
\[
v_m[\theta] = e^{-2\pi i m\theta/N}, \qquad m=0,\dots,N-1,
\]
which constitute the columns of the DFT matrix $F$. \footnote{
From a representation-theoretic viewpoint, this result follows directly from
the fact that circulant operators commute with the regular representation of
the cyclic group, and are therefore simultaneously diagonalizable in the Fourier
basis. The explicit computation above merely identifies the corresponding
eigenvalues.
}

Since $C_a$ commutes with all shifts, it is diagonal in the same Fourier basis.
It therefore suffices to compute the corresponding eigenvalues.

Let $v_m$ denote the $m$-th Fourier mode. A direct computation yields
\[
(C_a v_m)[\theta]
=
\sum_{n=0}^{N-1} a_n\, v_m[(\theta-n)\bmod N]
=
\left(\sum_{n=0}^{N-1} a_n e^{2\pi i mn/N}\right) v_m[\theta].
\]

The scalar factor in parentheses is precisely the discrete Fourier transform
coefficient $\hat a[m]$. Hence,
\[
C_a v_m = \hat a[m]\, v_m.
\]

Collecting all eigenvectors into the DFT matrix $F$ yields the diagonalization
\[
C_a = F^{-1}\,\mathrm{diag}(\hat a)\,F,
\]
which completes the proof.
\end{proof} 

Circular convolution with kernel $a$ can be represented as the action of the
circulant matrix $C_a$, so that
\[
a\otimes x = C_a x.
\]
Since $C_a$ is diagonalized by the DFT matrix $F$, we have
\[
C_a = F^{-1}\,\mathrm{diag}(\hat a)\,F.
\]
Applying the Fourier transform to the convolution yields
\[
\widehat{(a\otimes x)}
=
F C_a x
=
F F^{-1}\,\mathrm{diag}(\hat a)\,F x
=
\mathrm{diag}(\hat a)\,\hat x,
\]
which shows that circular convolution corresponds to pointwise multiplication
in the Fourier domain:
\[
\widehat{(a\otimes x)}[m]=\hat a[m]\;\hat x[m].
\]

\begin{remark}[Spectral structure of the rotor algebra]
The family of rotors $\{R_k\}_{k=0}^{N_\theta-1}$ is simultaneously diagonalizable.
All rotors share the same eigenvectors, given by the discrete Fourier modes, and
differ only in their eigenvalues, which are unit-modulus phase factors.

As a consequence, the Fourier basis defines the unique spectral directions of
the rotor algebra: angular rotations do not introduce new directions, nor do
they mix components, but act independently on each frequency.
In this sense, the space of rotors is fully characterized by its action along
these fixed spectral directions.
\end{remark}

\section{Spectral Form of Polar Convolution Operators}

Let $K\in\mathbb{R}^{N_r\times N_\theta}$ be a polar kernel and define the
convolutional operator $T_K(A)=A\otimes K$. For each fixed radius $r$, the action
reduces to circular convolution on the angular slice. As a consequence,
the diagonalization result~\eqref{eq:circulant_diag} applies independently at
each $r$.

Equivalently, letting $\hat A[r,m]$ and $\hat K[r,m]$ denote the angular DFTs,
the operator $T_K$ acts diagonally in the Fourier domain:
\[
\widehat{T_K(A)}[r,m] = \hat A[r,m]\;\hat K[r,m].
\]

This diagonal structure implies that each angular Fourier mode evolves
independently. In particular, whenever $\hat K[r,m]\neq 0$, the action of
$T_K$ can be inverted by pointwise division in the spectral domain. This
observation provides the spectral mechanism underlying the inversion result
stated in Aurora's Theorem.

\section{Rotors in the Fourier Domain}

Recall the rotor (shift) operator
\[
(R_kA)[r,\theta]=A[r,(\theta-k)\bmod N_\theta],
\qquad k\in\mathbb{Z}_{N_\theta}.
\]
We denote the identity operator by $\mathrm{Id}$, hence $R_0=\mathrm{Id}$.

Since $R_k$ is a circulant operator on each angular slice, it is diagonalized by
the DFT:
\[
R_k = F^{-1}\,\mathrm{diag}\!\left(e^{-2\pi i mk/N_\theta}\right)_{m=0}^{N_\theta-1}\,F.
\]
Thus, angular rotations do not mix spectral components; they only modulate
phases frequency-by-frequency.

\section{Equivariance Forces Rotor Expansions}

\subsection{Rotation Equivariance}
\begin{definition}[Rotation Equivariance]
A polar linear operator $\mathcal{T}$ is rotation-equivariant if it commutes with
all rotors:
\[
\mathcal{T}(R_kA)=R_k(\mathcal{T}A)
\qquad \forall k\in\mathbb{Z}_{N_\theta}.
\]
Equivalently, $\mathcal{T}R_k=R_k\mathcal{T}$.
\end{definition}

\subsection{Rotors Generate All Equivariant Operators}

\begin{proposition}[Rotor expansion of equivariant polar operators]
Let $\mathcal{T}$ be a linear operator acting on polar tensors
$A\in\mathbb{R}^{N_r\times N_\theta}$, such that:
\begin{itemize}
  \item $\mathcal{T}$ acts independently on each radius $r$;
  \item $\mathcal{T}$ is rotation-equivariant, so that \ $\mathcal{T}R_k = R_k\mathcal{T}$
        for all circular shifts $R_k$.
\end{itemize}
Then there exist coefficients $\{c_k[r]\}_{k=0}^{N_\theta-1}$, possibly depending
on the radius $r$, such that
\[
(\mathcal{T}A)[r,\theta]
=
\sum_{k=0}^{N_\theta-1} c_k[r]\,(R_k A)[r,\theta].
\]
\end{proposition}

\begin{proof}
Fix a radius $r$ and restrict $\mathcal{T}$ to the angular space
$\mathbb{R}^{N_\theta}$ at that radius. Under the independence assumption,
this restriction defines a linear operator $T_r:\mathbb{R}^{N_\theta}\to
\mathbb{R}^{N_\theta}$.

Rotation-equivariance implies that $T_r$ commutes with all circular shift
operators $R_k$:
\[
T_r R_k = R_k T_r \qquad \forall k\in\mathbb{Z}_{N_\theta}.
\]
Since the shifts are generated by the one-step shift $R_1$, it is sufficient
to require $T_r R_1 = R_1 T_r$.

Let $\{e_\theta\}$ denote the canonical basis of $\mathbb{R}^{N_\theta}$ and write
$t^{(r)}_\theta := T_r e_\theta$. Then
\[
t^{(r)}_{\theta+1}
=
T_r e_{\theta+1}
=
T_r (R_1 e_\theta)
=
R_1 (T_r e_\theta)
=
R_1 t^{(r)}_\theta.
\]
Thus all columns of $T_r$ are obtained by successive circular shifts of the first
column $t^{(r)}_0$. It follows that the matrix of $T_r$ is circulant, with entries
depending only on the index difference modulo $N_\theta$.

Therefore $T_r$ represents circular convolution with some kernel
$c[r]=\{c_k[r]\}_{k=0}^{N_\theta-1}$:
\[
(T_r a)[\theta]
=
\sum_{k=0}^{N_\theta-1} c_k[r]\,
a[(\theta-k)\bmod N_\theta].
\]
Rewriting this expression in operator form yields
\[
T_r = \sum_{k=0}^{N_\theta-1} c_k[r]\,R_k.
\]
Since the argument holds independently for each radius $r$, the full operator
$\mathcal{T}$ admits the stated rotor expansion.
\end{proof}

\paragraph{Interpretation.}
Rotation equivariance imposes a strong algebraic constraint: any linear operator
that commutes with the cyclic group action must belong to the commutative algebra
generated by the rotors $\{R_k\}$. Equivariance therefore forces convolution, and
the rotor basis provides a complete and minimal parametrization of all such
operators.

\section{Chiara's Theorem: Polar Self-Adjointness and Real Spectra}

Up to this point, we have analyzed specific classes of operators, such as rotors
and circulant operators, by explicitly constructing their action and studying their
spectral properties. While these examples already suggest a common underlying
structure, the results have been presented through individual cases and concrete
representations.

The following theorem provides a unifying perspective. Rather than introducing
new operators, it formalizes the shared spectral structure implicit in all the
previous constructions, allowing the analysis to move from explicit operator
definitions to a more abstract and general framework.

Recall that $A\in\mathbb{R}^{N_r\times N_\theta}$ is polar self-adjoint if
$A^{T_p}=A$, so that \ $A[r,\theta]=A[r,(-\theta)\bmod N_\theta]$.

\begin{theorem}[Chiara's Theorem]\label{theorem:chiara_theorem}
Let $A\in\mathbb{R}^{N_r\times N_\theta}$ satisfy $A^{T_p}=A$. Then, for each
radius $r$, the associated circulant operator is diagonalized by the DFT:
\[
C_{a_r}=F^{-1}\Lambda_r F,
\]
where $\Lambda_r=\mathrm{diag}(\hat a_r)$ is real-valued.
\end{theorem}

\begin{proof}
For each $r$, $C_{a_r}$ is circulant, hence diagonalized by $F$ with diagonal
$\mathrm{diag}(\hat a_r)$. Polar self-adjointness implies $a_r$ is real and even,
so $\hat a_r$ is real-valued. Therefore $\Lambda_r$ is real.
\end{proof}

Although the result may appear as a natural extension of the properties previously
established for rotors and circulant operators, Theorem~\ref{theorem:chiara_theorem} plays a distinct role.
Rather than introducing a new class of operators, it provides a unifying structural
statement that subsumes all previous cases into a single spectral framework.
From this point on, we will rely on this characterization without further reference
to individual constructions.
The theorem is formulated in abstract operator-theoretic terms but its
assumptions naturally arise in a wide range of practical settings, including
convolutional filtering, correlation operators, and translationally invariant
physical systems. In these cases, the theorem provides a rigorous justification
for the ubiquitous use of Fourier-based representations.

\section{Elisa's Theorem and Simultaneous Diagonalization}

By Elisa's Theorem (Chapter~2), the polar product is commutative. For each fixed
radius $r$, this implies that the angular convolution operators induced by polar
kernels commute:
\[
C_a^{(r)} C_b^{(r)} = C_b^{(r)} C_a^{(r)}.
\]

Since circulant operators are normal and commute, they admit a common orthonormal
eigenbasis. In the circulant case, this basis is explicitly given by the Fourier
modes. Consequently, the entire commutative algebra of polar convolution
operators is simultaneously diagonalized by the DFT and reduces to pointwise
multiplication in the frequency domain.

\section{Computational Complexity as a Spectral Consequence}

\begin{proposition}[Complexity of the Polar Product]
Naive polar convolution costs $O(N_rN_\theta^2)$, whereas FFT-based computation
costs $O(N_rN_\theta\log N_\theta)$.
\end{proposition}

\begin{proof}
Per radius, naive convolution costs $O(N_\theta^2)$. FFT diagonalization costs
$O(N_\theta\log N_\theta)$ for the forward/inverse transforms plus $O(N_\theta)$
for pointwise multiplication, hence $O(N_\theta\log N_\theta)$. Multiply by $N_r$.
\end{proof}

\section{Summary}
Polar convolution operators are diagonalized (radius-wise) by the DFT, rotors act
as frequency-dependent phase shifts, and rotation equivariance forces operators
to lie in the commutative algebra generated by the rotors. FFT-based computation
is therefore the canonical spectral realization of polar linear algebra.
At this point, the algebraic and spectral structure of polar operators is
fully characterized. In the following sections, we show how this structure
translates into efficient implementations and practical architectures.

\begin{figure}
\centering
\begin{tikzpicture}[
    node distance=2.6cm,
    every node/.style={draw, rectangle, rounded corners, align=center}
]
\node (elisa) {Elisa's Theorem\\Commutativity};
\node (aurora) [below left=1.8cm and 2.2cm of elisa] {Aurora's Theorem\\Invertibility};
\node (chiara) [below right=1.8cm and 2.2cm of elisa] {Chiara's Theorem\\Real spectrum in the self-adjoint case};

\draw[->] (elisa) -- (aurora);
\draw[->] (elisa) -- (chiara);
\end{tikzpicture}
\caption{Logical structure of three foundational results in polar linear algebra.
Elisa's Theorem provides the commutative spectral framework underlying the theory.
Within this framework, Aurora's Theorem gives the invertibility criterion for a general polar matrix, while Chiara's Theorem specializes to the self-adjoint case and establishes reality of the spectrum.}
\end{figure}
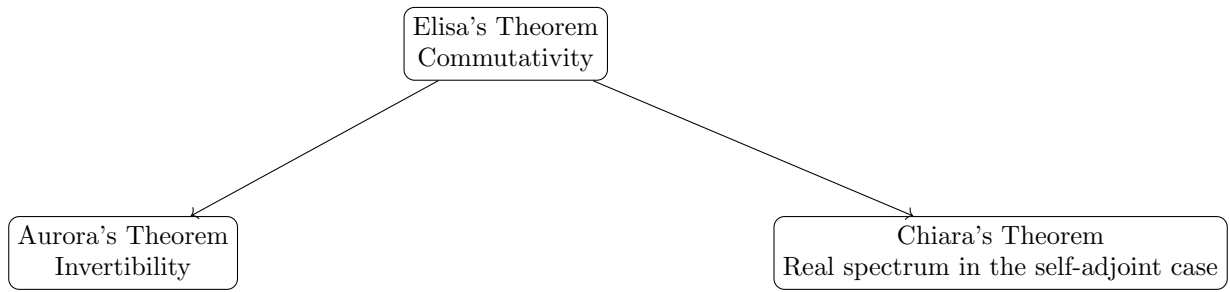

\chapter{Tensorial and Functional Representations}
\label{chap:implementation}

The previous chapters have developed an operator–theoretic and spectral
framework for polar matrices and angular transformations, focusing on their
algebraic structure, spectral properties, and diagonalization by the discrete
Fourier transform. 
The purpose of this chapter is not to introduce new spectral results or to
extend the theoretical analysis developed so far. Instead, its goal is to
recast the previously defined operators, signals, and transformations in a
tensorial and functional language that is compatible with learning-based
architectures. In particular, we move from the perspective of individual
operators acting on single objects to that of tensor-valued functions acting on
batched, multi-channel data.
The constructions introduced here serve as a
preparatory step for the fully spectral neural architectures discussed in later
chapters, where the tensorial formulation becomes essential for both
implementation and parallelization considerations.

\section{The PolarTensor Abstraction}
In order to translate the polar algebra developed in the previous chapters into
a numerical framework, we introduce the abstraction of a \emph{polar tensor}.

Formally, a polar tensor is an object of the form
\[
A[r,\theta_1,\dots,\theta_d]
\in \mathbb{C}^{N_r \times N_{\theta_1} \times \cdots \times N_{\theta_d}},
\]
where:
\begin{itemize}
\item $r$ denotes the radial index;
\item each $\theta_i$ denotes an angular index taking values in the cyclic group
$C_{N_{\theta_i}}$;
\item the radial dimension acts as a collection of independent channels;
\item angular dimensions are periodic and admit a Fourier representation.
\end{itemize}

\subsection{Angular Axes and Spectral Structure}

Let $\mathcal{A} = \{\theta_1,\dots,\theta_d\}$ denote the set of angular axes.
All spectral operations are performed exclusively along these axes, while the
radial axis is treated parametrically.

The discrete Fourier transform (DFT) of a polar tensor is defined as
\[
\widehat{A}[r,m_1,\dots,m_d]
=
\mathcal{F}_{\mathcal{A}}(A)[r,m_1,\dots,m_d],
\]
where $\mathcal{F}_{\mathcal{A}}$ denotes the $d$-dimensional DFT applied along
the angular axes.

This construction is the explicit numerical realization of the spectral
diagonalization proved in Chapter~3.

\subsection{Polar Product}

Given two polar tensors $A$ and $B$ with compatible dimensions, their
\emph{polar product} is implemented spectrally as
\[
A \otimes B
=
\mathcal{F}_{\mathcal{A}}^{-1}
\bigl(
\widehat{A}\;\widehat{B}
\bigr),
\]
where the multiplication in the Fourier domain is pointwise.
This formulation directly reflects Elisa's Theorem and the diagonalization
results of Chapter~3.

\subsection{Invertibility and Pseudo-Inverse}

By Aurora's Theorem, a polar tensor $A$ is invertible under the polar product if
and only if its Fourier coefficients are nowhere vanishing:
\[
\widehat{A}[r,m_1,\dots,m_d] \neq 0
\quad \forall r,\ (m_1,\dots,m_d).
\]

In numerical practice, invertibility is tested spectrally via thresholding.
When exact invertibility fails, a \emph{spectral pseudo-inverse} may be defined
by regularized inversion of the Fourier coefficients:
\[
\widehat{A}^{\dagger}
=
\frac{\widehat{A}^{\ast}}
{|\widehat{A}|^2 + \varepsilon},
\]
where $\varepsilon > 0$ is a small regularization parameter.
This construction parallels standard Tikhonov regularization and will play a
central role in learning-based applications discussed in Chapter~5.

From a learning perspective, the regularized spectral pseudo-inverse provides a
useful interpretation of the behavior observed during training in Chapter \ref{chapter:mnist_training} . Spectral modes
associated with large Fourier coefficients are effectively inverted and
preserved, while modes with small magnitude are automatically damped rather than
amplified. This mechanism explains the stability of the optimization process and
the absence of numerical instabilities, even when some spectral components are
close to singular. In the self-adjoint setting, where the spectrum is real-valued,
this effect becomes purely amplitude-based, further reinforcing the role of
energy-selective filtering in the learned representation.

\subsection{Rotors as Special Polar Tensors}

Angular rotors introduced in Chapter~2 admit a natural interpretation within
the polar tensor framework.

A \emph{rotor} is a polar tensor whose action corresponds to a discrete angular
shift along one or more angular axes. For a shift vector
$\mathbf{k} = (k_1,\dots,k_d)$, the associated rotor $R_{\mathbf{k}}$ satisfies
\[
(R_{\mathbf{k}} A)[r,\theta_1,\dots,\theta_d]
=
A[r,\theta_1-k_1,\dots,\theta_d-k_d].
\]

Spectrally, rotors are diagonal operators. In the Fourier basis, their action
reduces to a pointwise multiplication by a phase factor:
\[
(\widehat{R_{\mathbf{k}}}\,\hat A)[r,m_1,\dots,m_d]
=
e^{-i2\pi\left(
\frac{m_1 k_1}{N_{\theta_1}}+\cdots+\frac{m_d k_d}{N_{\theta_d}}
\right)}
\hat A[r,m_1,\dots,m_d].
\]

where each angular coordinate $\theta_j$ is discretized over
$N_{\theta_j}$ samples, with $\theta_j\in\mathbb{Z}_{N_{\theta_j}}$.

Thus, rotors appear as distinguished polar tensors whose spectra encode angular
translations. This viewpoint unifies the operator-based and kernel-based
interpretations of rotations within a single computational object.

In the present work we mostly consider the case $d=1$, corresponding to a single
angular coordinate $\theta\in\mathbb{Z}_{N_\theta}$, but the formulation extends
naturally to multiple angular dimensions.

\section{Python implementation}

The following reference implementation illustrates how the algebraic
constructions developed in the previous chapters translate directly into
numerical operations.
The code is intentionally minimal and mirrors the theoretical structure
one-to-one.

\begin{lstlisting}[style=polarcode, caption={PolarTensor abstraction for numerical polar algebra.}, label={lst:polartensor}]
import numpy as np

class PolarTensor:
    """
    Polar tensor with one radial axis and one or more angular axes.
    Angular axes are assumed to be periodic.
    """

    def __init__(self, data, angular_axes):
        self.data = np.asarray(data)
        self.angular_axes = tuple(angular_axes)
        self.radial_axis = 0

    def __repr__(self):
        return (
            f"PolarTensor(shape={self.data.shape}, "
            f"angular_axes={self.angular_axes}, "
            f"data=\n{self.data})"
        )

    def fft(self):
        """
        Discrete Fourier transform along all angular axes.
        """
        spectrum = np.fft.fftn(self.data, axes=self.angular_axes)
        return PolarTensor(spectrum, self.angular_axes)

    def ifft(self):
        """
        Inverse discrete Fourier transform along all angular axes.
        """
        spatial = np.fft.ifftn(self.data, axes=self.angular_axes)
        return PolarTensor(spatial, self.angular_axes)

    def polar_convolve(self, other):
        """
        Polar product implemented via spectral multiplication.
        """
        A_hat = self.fft()
        B_hat = other.fft()
        C_hat = A_hat.data * B_hat.data
        return PolarTensor(C_hat, self.angular_axes).ifft()

    def is_invertible(self, eps=1e-12):
        """
        Spectral invertibility test (Aurora's Theorem).
        """
        spectrum = self.fft().data
        return np.all(np.abs(spectrum) > eps)

    def pseudo_inverse(self, eps=1e-6):
        """
        Spectral pseudo-inverse via regularized reciprocal.
        """
        spectrum = self.fft().data
        inv_spec = np.conj(spectrum) / (np.abs(spectrum)**2 + eps)
        return PolarTensor(inv_spec, self.angular_axes).ifft()

    def apply_rotor(self, shifts):
        """
        Applies an angular shift (rotor) to the tensor.
        """
        shifted = self.data
        for ax, k in zip(self.angular_axes, shifts):
            shifted = np.roll(shifted, shift=k, axis=ax)
        return PolarTensor(shifted, self.angular_axes)
\end{lstlisting}

Although written in Python/NumPy for clarity, the implementation is fully
compatible with multidimensional FFT backends and parallel hardware
architectures, as discussed in Chapter~6.

\begin{lstlisting}[style=polarcode, caption={Rotor as a special polar tensor.}, label={lst:rotor}]
    @staticmethod
    def rotor(shape, angular_axes, shifts):
        """
        Constructs a rotor as a polar tensor.
        """
        data = np.zeros(shape, dtype=np.float64)

        idx = [slice(None)] * len(shape)
        for ax, k in zip(angular_axes, shifts):
            idx[ax] = k % shape[ax]

        data[tuple(idx)] = 1.0
        return PolarTensor(data, angular_axes)
\end{lstlisting}

\subsection{Outlook}

The polar tensor abstraction introduced here provides a minimal but complete
numerical counterpart to the polar algebra developed in the previous chapters.
Its design intentionally avoids restricting the number of angular dimensions,
thereby preparing the ground for applications to rotation-equivariant learning
and parallel spectral implementations on modern hardware
architectures.

\section{Dense Layers as Polar Operators}
\label{sec:dense_polar}

As a first concrete application of the polar tensor framework, we consider a
standard dense (fully connected) linear layer, one of the most common building
blocks in classical neural networks.

In its simplest form, a dense layer acts on a vector $x \in \mathbb{R}^N$ as
\[
y = Wx,
\]
where $W \in \mathbb{R}^{N \times N}$ is a learned weight matrix.

\subsection{Dense Layers in the Polar Domain}

When the input signal is represented as a polar tensor
\[
A[r,\theta_1,\dots,\theta_d],
\]
the dense layer can be interpreted as a linear operator acting independently
on each angular frequency.

Using the spectral representation introduced in Chapter~3, the action of a
dense layer becomes diagonal in the angular Fourier domain:
\[
\widehat{Y}[r,m_1,\dots,m_d]
=
W_r(m_1,\dots,m_d)\;
\widehat{A}[r,m_1,\dots,m_d],
\]
where $W_r$ denotes a set of learnable spectral weights, possibly dependent on
the radial index $r$.

Thus, a dense layer corresponds to a polar operator that performs
\emph{pointwise multiplication in the angular Fourier domain}, followed by an
inverse Fourier transform. No mixing between distinct angular frequencies
occurs.

\paragraph{Remark: Training in the Spectral Domain.}

In the polar dense layer of Listing~\ref{lst:polar_dense}, the learnable
parameters act directly on the angular Fourier coefficients.
As a result, training naturally takes place in the spectral domain.

This does not require any modification of the standard backpropagation
algorithm.
The discrete Fourier transform is a unitary linear operator, and therefore
preserves inner products, gradients, and optimization dynamics.
From the perspective of learning, the FFT and its inverse are simply fixed
changes of basis.

An important consequence is that each angular frequency defines an independent
optimization direction.
The learning problem decomposes into a collection of orthogonal subproblems,
one per frequency and radial channel, with no cross-frequency coupling.

This observation will be exploited explicitly in
Chapter~5, where training dynamics, stability, and inductive biases of polar
neural architectures are discussed in detail.

\subsection{Implementation as a Polar Operator}
From a numerical perspective, a dense angular layer may be implemented by:
\begin{enumerate}
\item transforming the polar tensor to the angular Fourier domain;
\item multiplying each frequency component by a learned weight;
\item transforming back to the spatial angular domain.
\end{enumerate}

This procedure provides an intuitive interpretation of the operation in terms
of spectral diagonalization, as established in Chapter~3, and is naturally
supported by the \texttt{PolarTensor} abstraction introduced in this chapter.
However, it is important to stress that the final inverse transform is not
strictly required. Once a convenient orthogonal basis has been identified, the
entire computation may be carried out directly in the Fourier domain. This
fully spectral formulation avoids
unnecessary transformations and leads to a more compact and efficient
implementation.

\subsection{Benefits of the Polar Formulation}

Interpreting dense layers as polar operators yields several advantages:
\begin{itemize}
\item \textbf{Structured parameterization}: weights act independently on each
angular frequency, reducing unnecessary coupling;
\item \textbf{Rotation equivariance}: the operator commutes with all angular
rotors by construction;
\item \textbf{Parallelism}: all frequency components are independent and can be
processed in parallel;
\item \textbf{Spectral interpretability}: learned weights admit a direct
frequency-domain interpretation.
\end{itemize}

These properties make polar dense layers a natural building block for
rotation-aware and rotation-equivariant neural architectures, as discussed in
the next chapter.

\begin{lstlisting}[style=polarcode,
caption={Dense layer implemented as a polar operator.},
label={lst:polar_dense}]
class PolarDense:
    """
    Dense linear layer acting spectrally on angular dimensions.
    """

    def __init__(self, weight):
        """
        Parameters
        ----------
        weight : ndarray
            Spectral weights with shape compatible with the angular spectrum.
        """
        self.weight = weight

    def __call__(self, A):
        """
        Applies the dense layer to a PolarTensor A.
        """
        A_hat = A.fft()
        Y_hat = A_hat.data * self.weight
        return PolarTensor(Y_hat, A.angular_axes).ifft()
\end{lstlisting}

\section{Complexity of Classical Dense Layers vs.\ Polar Dense Layers}
\label{sec:dense_comparison}

We now compare the computational structure of a classical dense layer with its
polar counterpart, focusing on both algebraic form and computational
complexity.

A classical dense layer acting on a vector $x \in \mathbb{R}^{N}$ is defined as
\[
y = Wx,
\]
where $W \in \mathbb{R}^{N \times N}$ is a fully populated weight matrix.

The computation of $y$ requires:
\begin{itemize}
\item $N^2$ multiplications,
\item $N(N-1)$ additions.
\end{itemize}
Hence, the time complexity of a single dense layer is
\[
T_{\text{dense}} = O(N^2).
\]

This cost reflects the fact that every output component depends on every input
component, with no exploitable structure.

Consider now a polar tensor
\[
A[r,\theta_1,\dots,\theta_d],
\]
and let $N_\theta = \prod_{i=1}^d N_{\theta_i}$ denote the total number of angular
samples.

A polar dense layer acts diagonally in the angular Fourier domain:
\[
\widehat{Y}[r,m_1,\dots,m_d]
=
W_r(m_1,\dots,m_d)\;
\widehat{A}[r,m_1,\dots,m_d],
\]
where $W_r$ denotes learned spectral weights.

The computation consists of:
\begin{enumerate}
\item one $d$-dimensional FFT along angular axes;
\item $N_\theta$ pointwise multiplications;
\item one inverse $d$-dimensional FFT.
\end{enumerate}

Using the FFT complexity estimate, the total cost per radial channel is
\[
T_{\text{polar}} = O\!\left(N_\theta \log N_\theta\right).
\]

Since all radial channels are independent, the full cost scales linearly with
the number of radial samples $N_r$.

For sufficiently large angular resolution $N_\theta$, the polar formulation
yields a substantial reduction in computational complexity.
Moreover, the polar dense layer exposes fine-grained parallelism across angular
frequencies, which is not present in the classical dense formulation.

The reduction from quadratic to quasi-linear complexity is not obtained through
approximation or sparsification, but follows directly from the algebraic
structure imposed by angular periodicity.
The FFT-based implementation is therefore a consequence of the polar algebra,
rather than an external optimization.

\chapter{Polar Neural Operators on MNIST}\label{chapter:mnist_training}

\section{Chapter Scope and Validation Strategy}

The purpose of this chapter is to provide a concrete validation of the
theoretical framework developed in the previous chapters.
Rather than introducing a new state-of-the-art architecture or proposing a
general replacement for existing deep learning models, the goal here is more
foundational and demonstrative.

Specifically, this chapter aims to show that the proposed polar--spectral
algebraic formulation:
\begin{itemize}
  \item can be translated into an explicit computational pipeline,
  \item can be implemented using existing deep learning frameworks without
        requiring custom runtimes or non-standard tooling,
  \item and can be trained end-to-end on a standard benchmark, achieving
        strong performance despite the simplicity of the setting.
\end{itemize}

The MNIST dataset is deliberately chosen as a simple and well-understood test
case.
Its role in this work is not to assess ultimate performance, but to act as a
minimal and controlled environment in which the feasibility of the proposed
operators can be evaluated.
In this sense, MNIST serves as a \emph{didactic} and \emph{vintage} benchmark,
allowing the focus to remain on the structure of the operators rather than on
data complexity.

It is important to emphasize what this chapter does \emph{not} attempt to do.
The objective is not to rewrite the entirety of deep learning models or
convolutional architectures within a polar or spectral paradigm.
Nor is the intent to claim that all operators should be expressed in polar
coordinates or that spectral formulations universally dominate spatial ones.
Instead, the contribution is limited and precise:
\begin{enumerate}
  \item the introduction of a new algebraic viewpoint for representing signals
        and operators,
  \item a formal mathematical description of the problem in polar and spectral
        terms,
  \item a concrete hypothesis for how common operators can be realized within
        this framework,
  \item and a proof-of-concept implementation validating these ideas on a
        simple classification task.
\end{enumerate}

In this chapter, we present a sequence of numerical experiments on MNIST designed to progressively validate and refine the proposed polar neural operator framework.
Rather than introducing a single final architecture, we deliberately follow an incremental approach, where each experiment addresses a specific conceptual question and motivates the subsequent design choices.

The experiments evolve from a minimal polar operator trained largely in the spatial domain, to a fully spectral formulation, and finally to a self-adjoint-inspired architecture explicitly biased toward real spectra and stable dynamics.
This progression reflects the development of the underlying research ideas and highlights the role of spectral structure in both expressivity and optimization.

All experiments presented in this chapter are implemented in a unified codebase,
which is publicly available and maintained at
\url{https://github.com/giovanniguasti/polar-algebra}
and archived on Zenodo.

\subsection*{Acknowledgment on Tool Support}

For the sake of transparency, the author notes that an indispensable contribution
to the debugging and stabilization of the fully spectral model was provided
through interactive assistance from ChatGPT.
In particular, the identification of silent failure modes, diagnostic strategies,
and iterative architectural refinements benefited significantly from this support.

It should also be emphasized that early versions of the fully spectral model
described in this work did not converge or exhibited severe accuracy plateaus.
The final results were achieved only after extensive debugging and systematic
analysis of spectral energy flow.

\section{FFT-Based Implementation of Polar Operators}

The numerical formulation developed in Chapter~4 shows that polar convolution
operators admit a particularly convenient representation in the Fourier domain,
where they become diagonal (frequency-wise decoupled).
This observation primarily concerns the \emph{implementation} of polar
operators, rather than requiring new learning rules:
standard gradient-based optimization remains unchanged, but operator evaluation
can exploit FFT-based diagonalization.

Recall from Chapter~3 that, for a fixed radius $r$, the polar convolution
\[
(A \otimes K)[r,\theta]
\]
becomes, in the angular Fourier domain,
\[
\widehat{(A \otimes K)}[r,m]
=
\hat A[r,m]\;\hat K[r,m],
\]
where the product is pointwise over angular frequencies.
As a consequence, a polar linear layer can be interpreted as a collection of
independent linear transformations acting on orthogonal spectral components.
The learnable parameters of the layer are the spectral coefficients
$\hat K[r,m]$, rather than spatial-domain kernel values.

This choice also yields a significant reduction in the number of learnable parameters.
In the spatial domain, a convolutional kernel of size $k_r \times k_\theta$ requires
$C_{\mathrm{in}} C_{\mathrm{out}} k_r k_\theta$ parameters.
In contrast, the spectral formulation learns only
$C_{\mathrm{in}} C_{\mathrm{out}} R (2M+1)$ coefficients, where $R$ and $M$ denote the number of retained radial and angular Fourier modes.

\section{Advantages of Spectral Evaluation}

Evaluating polar operators in the spectral domain yields several advantages:

\paragraph{Implicit regularization.}
Spectral representations make it natural to constrain or penalize high-frequency
components, encouraging smooth and structured operators.

\paragraph{Computational efficiency.}
FFT-based implementations reduce complexity from quadratic to quasi-linear in
the angular resolution (up to constant factors and truncation choices).

\paragraph{Interpretability.}
Learned parameters correspond to frequency responses, enabling direct inspection
of what the model learns at each radius and angular scale.

\paragraph{Model size}
Since $R$ and $M$ are typically small and independent of the spatial resolution, the spectral operator achieves global angular support with substantially fewer parameters.
This parameter efficiency becomes increasingly pronounced as the resolution grows or when long-range angular interactions are required.
The self-adjoint-inspired constraints further reduce the effective degrees of freedom, acting as an additional inductive bias.

\paragraph{Validity of Gradient-Based Training}

The discrete Fourier transform is a linear, invertible, and unitary operator.
As such, it preserves inner products, norms, and linear structure.
Training a model that includes Fourier transforms is therefore equivalent to
training an isomorphic model in the spatial domain under a change of
coordinates.

More explicitly, let $\mathcal{L}$ be a loss function defined on the network
output.
If the forward pass includes a Fourier transform, the backward pass propagates
gradients through the inverse transform.
Since the transform is linear, differentiation commutes with it, and the
standard backpropagation equations remain valid without modification.

No new learning rules are required: gradients with respect to spectral
parameters are computed exactly as gradients with respect to spatial
parameters, while benefiting from the diagonal structure of the operator.

\paragraph{Stability of Learning in Spectral Models}
\label{sec:learning_stability}

In learning-based architectures, the explicit computation of inverse operators
is rarely required. Nevertheless, the notion of invertibility plays a fundamental
role in the stability of training and in the preservation of information across
layers. Unitary transforms preserve norms, reducing gradient explosion or
vanishing.\footnote{
A unitary transform $U$ satisfies $U^\ast U = I$ and therefore preserves the
$\ell_2$ norm, so that \ $\|Ux\|_2 = \|x\|_2$.
This property holds both in the forward pass and during backpropagation, since
gradients propagate through $U$ and its adjoint $U^\ast$ without additional
scaling.
In deep networks, exploding or vanishing gradients typically arise from the
repeated composition of operators whose norms are respectively larger or
smaller than one.
By contrast, unitary transforms have condition number one and do not introduce
systematic amplification or attenuation, improving numerical stability by
ensuring that changes in gradient magnitude originate from learnable or
nonlinear operators rather than from the representation change itself.
} In particular, the behavior of gradients during backpropagation is
strongly influenced by the spectral properties of the linear operators composing
the model.

In fully spectral architectures, convolutional layers reduce to multiplicative
operations in the frequency domain. If the corresponding spectral symbol vanishes
or becomes excessively small over subsets of the spectrum, entire components of
the signal may be irreversibly suppressed. This leads not only to information
loss in the forward pass, but also to degraded or vanishing gradients during
training.

Aurora’s Theorem provides a structural criterion to prevent such pathological
cases. By characterizing the invertibility of polar convolution operators in
terms of their spectral representation, it ensures that no relevant subspace of
the signal is collapsed by construction. From a learning perspective, this
property translates into improved numerical conditioning and more stable gradient
propagation.

It is therefore important to emphasize that invertibility is not introduced here
as an operational requirement but as a design principle. Infact no explicit inverse is computed during training. Ensuring that spectral operators remain well-posed
across the active frequency range contributes to robustness, smoother
optimization dynamics, and reduced sensitivity to initialization and learning
rate choices.

\paragraph{Orthogonality and Decoupling of Learning Dimensions}
Learnable parameters act multiplicatively on orthogonal spectral components,
reducing unnecessary parameter interactions.

In the spectral representation, each angular frequency corresponds to an
orthogonal basis vector.
As shown in Chapter~3, polar convolution operators act independently on each
frequency component.

This has a direct consequence for training: the learning problem decomposes
into a collection of independent subproblems, one per angular frequency.
There is reduced interference between frequencies, and parameter updates at one
frequency do not directly affect others.

From an optimization perspective, this reduces cross-coupling between
parameters and improves conditioning.
From a modeling perspective, it allows the network to learn frequency-selective
responses in a controlled and interpretable manner. 

\paragraph{Outlook: Spectral Parallelism.}
The decoupling of angular frequencies suggests a form of parallelism that goes
beyond standard data, pipeline, or expert parallelism.
Since each frequency component can be processed independently within FFT-based
operators, polar networks admit a natural \emph{spectral tensor parallelism},
in which disjoint groups of frequencies can be assigned to different
processing units.
This structure will be revisited in the next chapter.

\paragraph{Consequences for Deep Architectures}

In multi-layer polar networks, the spectral structure informs both model design
and implementation.
Frequency-wise decoupling supports:
\begin{itemize}
\item deep stacks of polar linear operators,
\item rotation-aware convolutional architectures,
\item frequency-selective mechanisms,
\item and efficient parallelization across frequencies and radii.
\end{itemize}

These aspects motivate the discussion of hardware implications and parallel
implementations in the next chapter.

\subsection{Staying in the Spectral Domain} \label{sec:staying_spectral}

The previous discussion characterizes polar linear operators through their
Fourier-domain diagonalization.
It is important to clarify that \emph{remaining entirely in the spectral domain}
is possible only under specific architectural conditions.

When a network is composed exclusively of \emph{linear} polar operators that
are simultaneously diagonalized by the same Fourier transform (e.g., sequences
of polar convolution/dense layers without spatial-domain nonlinearities or
normalization), the composition remains diagonal in the spectral domain.
In that case, one forward transform at the input and, if needed, one inverse
transform at the output are sufficient.

Let $\{\mathcal{T}_\ell\}_{\ell=1}^L$ be a sequence of polar linear operators.
In the Fourier domain, each operator acts as
\[
\widehat{\mathcal{T}_\ell A}[r,m]
=
W_\ell[r,m]\;\hat A[r,m],
\]
and their composition satisfies
\[
\hat A_{\mathrm{out}}[r,m]
=
\Bigl(\prod_{\ell=1}^L W_\ell[r,m]\Bigr)\;\hat A_{\mathrm{in}}[r,m].
\]

In practical deep networks, however, normalization and nonlinear activations
are typically defined in the spatial domain.
As a result, implementations alternate between spectral evaluation of linear
operators and spatial-domain nonlinearities, requiring FFT and inverse FFT
within each block.
This is the design choice adopted in the MNIST proof-of-concept presented
later in this chapter, in order to preserve the semantics of standard nonlinear
operators.

\subsection{Activation Operator vs. Spectral Gate Operator}
Before analyzing the empirical results, it is useful to clarify how nonlinearity
is introduced in the proposed framework, and how this differs from conventional
neural architectures.

In conventional convolutional neural networks, nonlinearity is introduced
through pointwise activation functions, such as ReLU or SiLU, applied
independently to each spatial entry:
\[
y_i = \sigma(x_i).
\]
These activations operate locally and modulate the amplitude of individual
features, acting as value-based valves that selectively amplify or suppress
signal magnitudes. Their role is to break global linearity while preserving
spatial locality, enabling the composition of expressive hierarchical models.

In contrast, the polar spectral framework replaces pointwise activations with a
\emph{gate operator} acting along the angular dimension. Rather than operating
on individual values, the gate is defined as a circular convolution:
\[
Y_k = \frac{1}{N_\theta} \sum_m W_m \, X_{k-m},
\]
which introduces a structured, non-local modulation across angular coordinates.
The gate does not act on single entries, but on directional components, coupling
neighboring angles through a smooth, band-limited kernel.
Conceptually, this choice shifts the role of nonlinearity from individual values
to directional structures.
Despite these differences, the two operators occupy an analogous position within
their respective architectures. In both cases, the operator is responsible for
introducing nonlinearity between linear transformations, thereby preventing the
collapse of the model into a purely linear system. However, the nature of the
nonlinearity differs fundamentally: pointwise activations act as valves on
feature values, whereas spectral gates act as valves on directions.

From this perspective, the polar spectral gate can be interpreted as a
direction-selective analogue of classical activation functions. It enables the
model to emphasize or suppress entire angular components, rather than individual
spatial values, introducing an anisotropic and operator-valued form of
nonlinearity. This shift from value-based to direction-based modulation is a
direct consequence of the underlying polar geometry and represents a distinctive
feature of the proposed algebraic framework.

The originality of this approach lies in redefining the concept of activation in
a way that is consistent with operator learning and spectral representations.
Rather than adapting existing pointwise nonlinearities to a new domain, the
polar algebra framework introduces a structurally aligned alternative, tailored
to the geometry and decomposition of the problem itself.

\section{MNIST Proof-of-Concept}

\subsection{Baseline polar operator}

The first experiment serves as a feasibility study.
Here, we investigate whether a neural operator defined in polar coordinates can be successfully trained on MNIST using largely standard components.
The architecture retains a predominantly spatial processing pipeline, while introducing polar indexing and angular structure only at the operator level.

The example uses a model that operates entirely in
polar coordinates and relies on spectral operators inspired by Fourier Neural
Operators.
In particular, spectral evaluation is employed for linear operators, while
nonlinear activations are applied in the spatial polar domain, consistent with
the discussion in Section~\ref{sec:staying_spectral}.
The same framework naturally supports hybrid architectures, where polar and
non-polar operators coexist and are composed within a single pipeline.

This experiment demonstrates that polar representations do not inherently hinder optimization or accuracy.
Despite the minimal use of spectral structure, the model achieves competitive performance, establishing a baseline and validating the use of polar coordinates as a viable representation for neural operators.

While the first experiment confirms feasibility, it does not exploit the natural spectral structure induced by polar coordinates.
In particular, angular periodicity and radial frequency separation suggest a formulation directly in the Fourier domain.

\paragraph{Pipeline and Operator Instantiation}

Let $I(x,y) \in \mathbb{R}^{28 \times 28}$ be an input image.
The first step of the pipeline is a deterministic change of representation
from Cartesian to polar coordinates:
\[
\mathcal{P}: \mathbb{R}^{28 \times 28} \longrightarrow \mathbb{R}^{R \times T},
\qquad
A(r,\theta) = \mathcal{P}[I(x,y)] .
\]
This operator corresponds to the polar resampling discussed in the previous
chapters and is implemented via grid-based interpolation.
Since $\mathcal{P}$ is fixed and computationally expensive, its output is
cached on disk and reused across training epochs.

To compensate for scale differences across radial rings, a radial
equalization operator $\mathcal{E}$ is applied:
\[
A'(r,\theta) = \mathcal{E}[A](r,\theta).
\]

The need for this operation arises from the geometry of polar coordinates.
In a polar discretization with uniform sampling in $(r,\theta)$, each sample
at radius $r$ represents a spatial sector whose area scales proportionally
to $r$.
As a consequence, inner rings aggregate information from a significantly
smaller spatial region than outer rings, leading to systematically lower
variance and energy in the corresponding samples.

Without correction, this imbalance biases both activations and gradients
toward outer radii, particularly in architectures where radial bands are
processed independently, as in the present spectral formulation.
Radial equalization mitigates this effect by normalizing ring-wise statistics
and applying a mild radius-dependent weighting that compensates for the
underlying area distortion.

The weighting function is inspired by the geometric Jacobian of the polar
transform and aims to partially restore scale homogeneity across radii.
Its specific functional form is chosen conservatively to stabilize training
without over-amplifying noise in the innermost rings.

The equalized polar signal is lifted to a higher-dimensional feature space
using a pointwise linear operator:
\[
\mathcal{L}: \mathbb{R}^{1 \times R \times T}
\longrightarrow
\mathbb{R}^{W \times R \times T}.
\]
In practice, $\mathcal{L}$ is implemented as a $1 \times 1$ convolution.
The resulting real-valued features are then interpreted as complex-valued
fields by initializing the imaginary part to zero:
\[
\mathbf{z} = \mathbf{z}_{\Re} + i\,\mathbf{0},
\qquad
\mathbf{z} \in \mathbb{C}^{W \times R \times T}.
\]

\paragraph{Spectral Polar Operator and PolarFNO Blocks}

The core operator is a spectral convolution acting on polar feature maps.
Let $\mathcal{F}$ denote the two-dimensional Fourier transform over the
$(r,\theta)$ grid. A spectral convolution is defined as
\[
\widehat{\mathbf{z}}(k_r,k_\theta) =
\mathcal{F}[\mathbf{z}](k_r,k_\theta),
\]
followed by a mode-wise multiplication
\[
\widehat{\mathbf{y}}(k_r,k_\theta) =
\widehat{W}(k_r,k_\theta)\,
\widehat{\mathbf{z}}(k_r,k_\theta),
\]
where $\widehat{W}$ is a learnable complex-valued tensor defined on a truncated
set of low-frequency modes.
The inverse transform
\[
\mathbf{y} = \mathcal{F}^{-1}[\widehat{\mathbf{y}}]
\]
returns the signal to the polar spatial domain.
Each PolarFNO block combines the spectral operator with:

\begin{itemize}
\item a pointwise ($1\times 1$) convolution for channel mixing,
\item a residual connection,
\item complex-valued normalization, and
\item a magnitude-based nonlinearity (ModReLU).
\end{itemize}
Since normalization and activation are applied in the spatial polar domain,
each block performs a forward and inverse Fourier transform.


\begin{figure}
\centering
\begin{tikzpicture}[
  font=\small,
  node distance=7mm,
  block/.style={draw, rounded corners, align=center, inner sep=6pt, text width=0.82\linewidth},
  arrow/.style={-Latex, thick}
]

\node[block] (ds) {MNIST dataset\\ \footnotesize $28\times 28$ grayscale digits (0--9)};

\node[block, below=of ds] (split) {Deterministic split\\ \footnotesize train/val via fixed seed};

\node[block, below=of split] (polar) {Polar resampling $\mathcal{P}$\\ \footnotesize build $(R,T)$ polar grid and resample with \texttt{grid\_sample}};

\node[block, below=of polar] (cache) {On-disk cache\\ \footnotesize store/load cached polar tensors (\texttt{.pt}) for train/val/test};

\node[block, below=of cache] (loader) {DataLoader\\ \footnotesize batching + shuffling (train)};

\node[block, below=of loader] (model) {PolarFNO model\\ \footnotesize radial equalization + complex PolarFNO blocks + classifier head};

\node[block, below=of model] (train) {Training loop\\ \footnotesize cross-entropy, Adam, gradient clipping};

\node[block, below=of train] (val) {Validation\\ \footnotesize loss/accuracy + best checkpoint selection};

\node[block, below=of val] (test) {Test evaluation\\ \footnotesize final accuracy on held-out test set};

\node[block, below=of test] (art) {Artifacts\\ \footnotesize checkpoints (last/best) + training curves};

\draw[arrow] (ds) -- (split);
\draw[arrow] (split) -- (polar);
\draw[arrow] (polar) -- (cache);
\draw[arrow] (cache) -- (loader);
\draw[arrow] (loader) -- (model);
\draw[arrow] (model) -- (train);
\draw[arrow] (train) -- (val);
\draw[arrow] (val) -- (test);
\draw[arrow] (test) -- (art);

\end{tikzpicture}
\caption{End-to-end MNIST pipeline used in the proof-of-concept implementation:
polar resampling (cached on disk), training/validation loops, and exported artifacts.}
\label{fig:mnist_pipeline_vertical}
\end{figure}
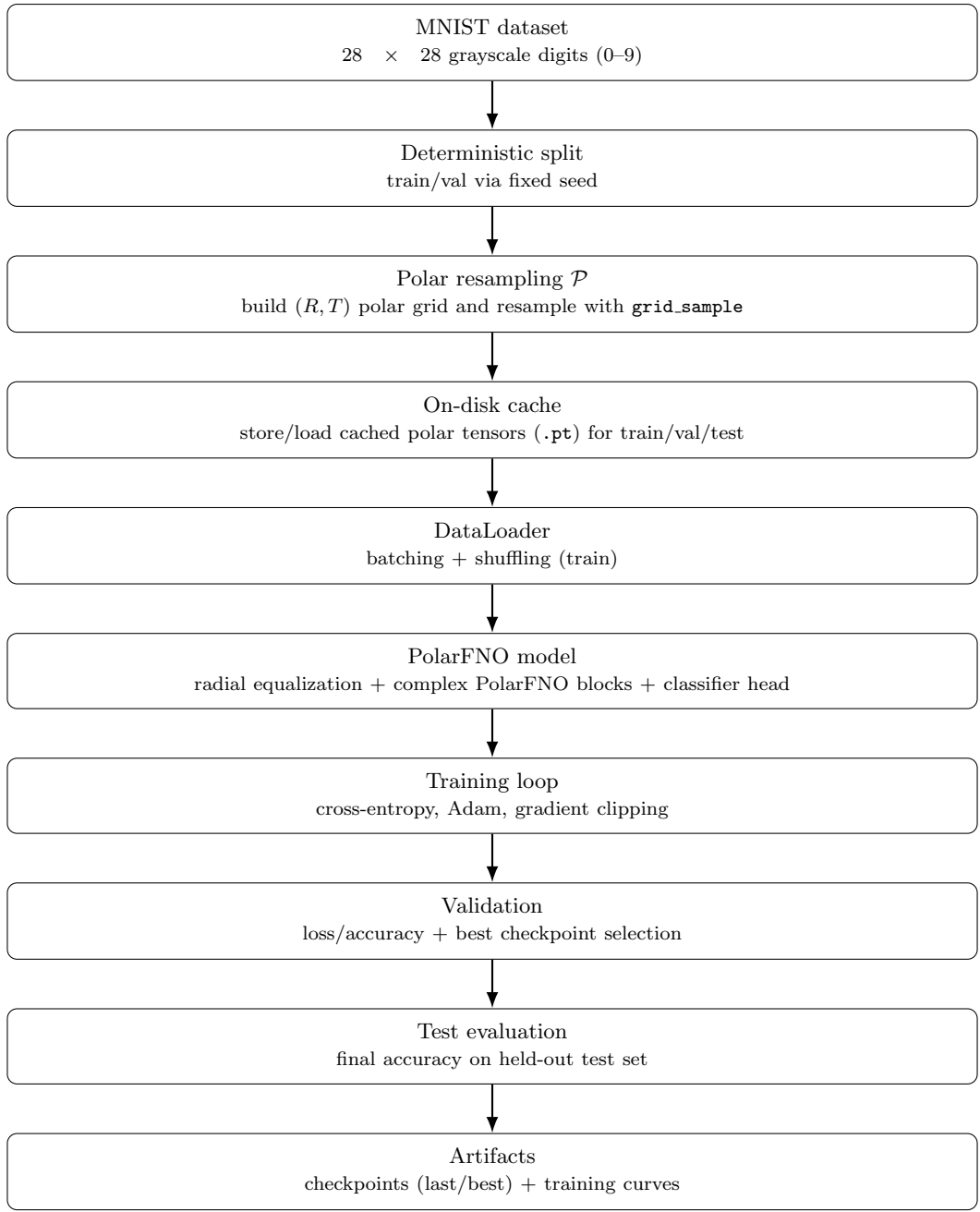


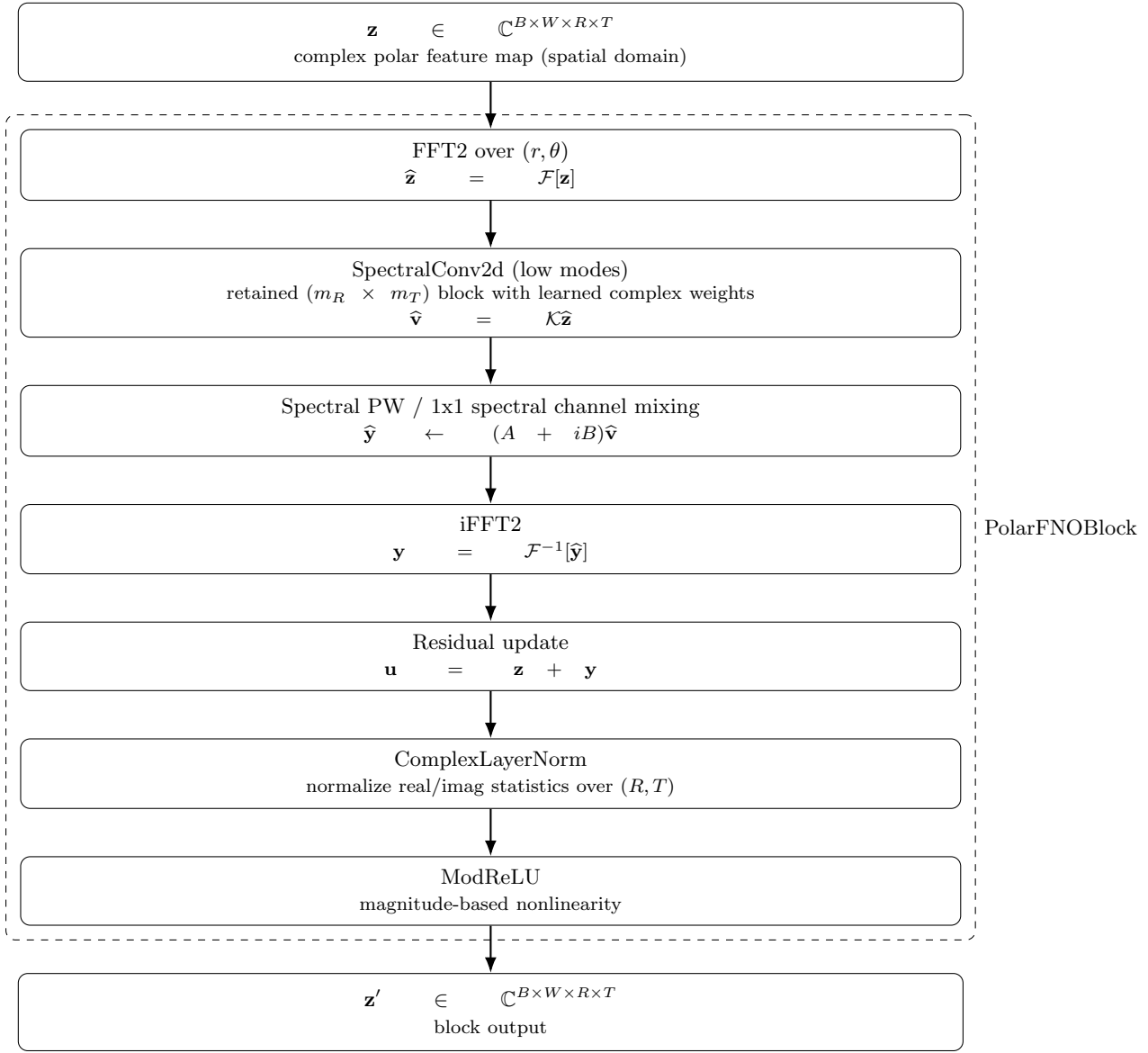
\begin{figure}
\centering
\begin{tikzpicture}[
  font=\small,
  node distance=7mm,
  block/.style={draw, rounded corners, align=center, inner sep=6pt, text width=0.82\linewidth},
  small/.style={draw, rounded corners, align=center, inner sep=5pt, text width=0.82\linewidth},
  arrow/.style={-Latex, thick}
]

\node[block] (zin) {$\mathbf{z}\in\mathbb{C}^{B\times W\times R\times T}$\\ \footnotesize complex polar feature map (spatial domain)};

\node[small, below=of zin] (fft) {FFT2 over $(r,\theta)$\\ \footnotesize $\widehat{\mathbf{z}}=\mathcal{F}[\mathbf{z}]$};

\node[small, below=of fft] (spec) {SpectralConv2d (low modes)\\ \footnotesize retained $(m_R\times m_T)$ block with learned complex weights\\ \footnotesize $\widehat{\mathbf{v}}=\mathcal{K}\widehat{\mathbf{z}}$};

\node[small, below=of spec] (pw) {Spectral PW / 1x1 spectral channel mixing\\ \footnotesize $\widehat{\mathbf{y}} \leftarrow (A+iB)\widehat{\mathbf{v}}$};

\node[small, below=of pw] (ifft) {iFFT2\\ \footnotesize $\mathbf{y}=\mathcal{F}^{-1}[\widehat{\mathbf{y}}]$};

\node[small, below=of ifft] (res) {Residual update\\ \footnotesize $\mathbf{u}=\mathbf{z}+\mathbf{y}$};

\node[small, below=of res] (norm) {ComplexLayerNorm\\ \footnotesize normalize real/imag statistics over $(R,T)$};

\node[small, below=of norm] (act) {ModReLU\\ \footnotesize magnitude-based nonlinearity};

\node[block, below=of act] (zout) {$\mathbf{z}'\in\mathbb{C}^{B\times W\times R\times T}$\\ \footnotesize block output};

\draw[arrow] (zin) -- (fft);
\draw[arrow] (fft) -- (spec);
\draw[arrow] (spec) -- (pw);
\draw[arrow] (pw) -- (ifft);
\draw[arrow] (ifft) -- (res);
\draw[arrow] (res) -- (norm);
\draw[arrow] (norm) -- (act);
\draw[arrow] (act) -- (zout);

\node[draw, dashed, rounded corners, fit=(fft)(spec)(pw)(ifft)(res)(norm)(act),
      inner sep=6pt, label={[font=\small]right:PolarFNOBlock}] {};

\end{tikzpicture}
\caption{PolarFNOBlock used in the MNIST experiment: the spectral operator is applied on retained low-frequency modes. The spectral pointwise path is the frequency-space analogue of a 1x1 convolution, performing channel mixing on the intermediate spectral tensor. Normalization and nonlinearity are applied in the spatial domain, thus requiring one FFT/iFFT pair per block.}
\label{fig:polarfno_block_spectral_pw}
\end{figure}

\paragraph{Operator-level excerpts.}
For readability, this chapter reports only the core operators used in the
MNIST proof-of-concept (polar resampling, radial equalization, and spectral
convolution).
The reference implementation is publicly available\footnote{\url{https://github.com/giovanniguasti/polar-algebra}.}

\paragraph{Model compactness.}
We observed that the proposed polar--spectral architecture remains competitive
even at reduced channel width.
For instance, using $W=32$ and only $3$ PolarFNO blocks, the model reaches
$95.0\%$ validation accuracy and $94.9\%$ test accuracy on MNIST after two training epochs.
This indicates that, for this benchmark, a compact set of low-frequency polar
modes combined with spectral channel mixing is sufficient to achieve strong
performance, supporting the practical viability and parameter efficiency of the
proposed operator formulation.

\begin{table}[t]
\centering
\begin{tabular}{c c c c c c}
\hline
$W$ & blocks & $m_R$ & $m_T$ & Val Acc (\%) & Test Acc (\%) \\
\hline
32 & 3 & 10 & 16 & 94.4 & 94.5 \\
32 & 2 & 12 & 20 & 92.4 & 92.7 \\
32 & 3 & 10 & 20 & 95.0 & 94.9 \\
\hline
\end{tabular}
\caption{Effect of channel width, depth, and retained spectral modes on MNIST.
Depth has a stronger impact than the number of retained modes, while moderately
increasing $m_T$ improves performance when depth is sufficient.}
\label{tab:mnist_ablation}
\end{table}

\paragraph{Training Curves and Performance}

Figure~\ref{fig:mnist_training_curves} reports training and validation curves
(loss and accuracy), illustrating stable convergence and strong performance
for this vintage benchmark.

\begin{figure}
    \centering
    \includegraphics[width=1\linewidth]{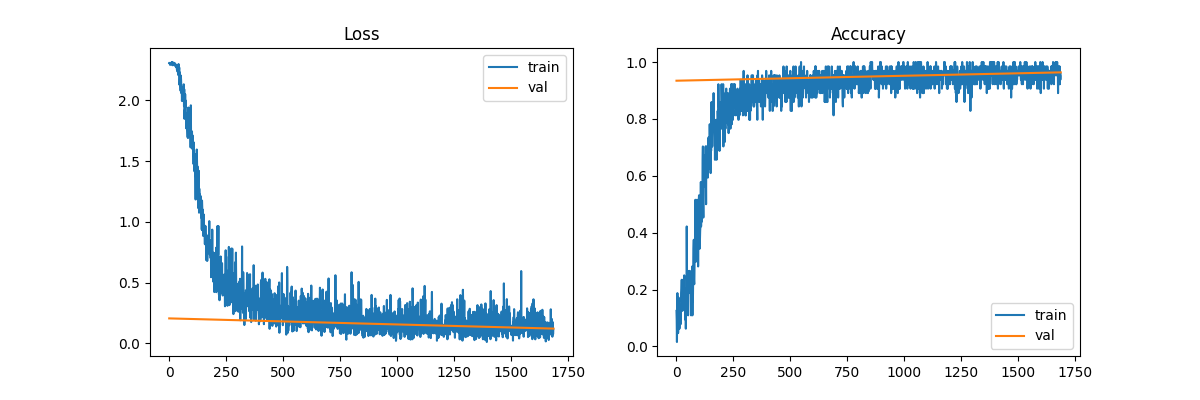}
    \caption{Training and validation curves (loss and accuracy) for the MNIST PolarFNO proof-of-concept. The horizontal axis tracks training progression, and the vertical axes show the corresponding loss and accuracy values.}
    \label{fig:mnist_training_curves}
\end{figure}

Having established that polar operators can be trained effectively in a largely
spatial pipeline, we next move to a fully spectral formulation, where both the
operator and the nonlinearity are expressed entirely in the frequency domain.

\subsection{Fully spectral polar architecture}
In the second experiment, we reformulate the polar operator entirely in the spectral domain.
Both radial and angular dimensions are processed using Fourier representations, allowing convolution, channel mixing, and gating to be expressed directly in frequency space.
The main objective of this experiment was to eliminate repeated transitions between
the spatial domain and the frequency domain inside the network, achieving a
fully spectral pipeline of the form
\[
\text{FFT} \;\rightarrow\; \text{spectral blocks} \;\rightarrow\; \text{spectral head},
\]
without intermediate inverse FFTs (IFFT).

This design aims to:
\begin{itemize}
  \item reduce computational overhead caused by repeated FFT/IFFT operations;
  \item maintain mathematical consistency with a fully spectral interpretation of the model;
  \item investigate whether low-frequency Fourier coefficients alone can support accurate learning.
\end{itemize}

This fully spectral formulation aligns more naturally with the underlying geometry of the polar grid and removes unnecessary spatial-domain operations.
Empirically, the spectral model improves efficiency and stability, confirming that the polar representation is particularly well-suited to spectral processing.

In the MNIST experiment, the operators employed are not explicitly constrained
to be self-adjoint in the classical sense of real symmetric circulant matrices.
Nevertheless, the architecture operates entirely within a Fourier-diagonalizable
framework, where translational invariance and spectral decoupling are preserved.
As a result, the effective operators exhibit stability properties consistent with
those predicted by the theory of self-adjoint circulant operators.

\paragraph{Main Challenges}
Moving the entire model to the frequency domain introduced several non-trivial challenges:

\begin{enumerate}
  \item \textbf{Lack of natural nonlinearities in frequency space.}
  Standard activation functions such as ReLU, SiLU, or GELU are defined in the
  spatial domain and cannot be directly applied in frequency space without
  breaking consistency.
  
  \item \textbf{Intrinsic periodicity of the angular dimension.}
  Functions defined along the angular coordinate $\theta$ must be periodic,
  while most standard activations are not.
  
  \item \textbf{Spectral energy control.}
  Unconstrained spectral operations can amplify certain modes, leading to
  numerical instability or saturation of the classifier head.
  
  \item \textbf{Limited observability of failure modes.}
  Architectural issues often do not cause divergence but instead result in
  silent accuracy plateaus that are difficult to diagnose.
\end{enumerate}

\paragraph{Spectral Architecture}
The model is built upon Fourier Neural Operator (FNO) blocks adapted to polar
coordinates $(r, \theta)$.

Each block operates entirely in the frequency domain and includes:
\begin{itemize}
  \item spectral convolutions restricted to low-frequency modes;
  \item residual connections with a learnable gain parameter $\alpha$;
  \item spectral RMS normalization;
  \item a \textbf{periodic spectral gate} acting along the angular dimension.
\end{itemize}

The final output is not transformed back to the spatial domain; instead,
classification is performed directly on spectral features extracted from
low-frequency quadrants.

\paragraph{Periodic Spectral Gate}
To introduce a nonlinearity compatible with angular periodicity, a spectral
gate is defined as a circular convolution:
\[
Y_k = \frac{1}{N_\theta} \sum_m W_m \, X_{k-m},
\]
where $W(\theta)$ is a periodic function, parameterized in the Fourier domain
(by coefficients $\{W_m\}$) and defined so as to be:
\begin{itemize}
  \item periodic by construction;
  \item spectrally regularized via Tukey windowing, which progressively attenuates high-frequency components;
  \item shaped to approximate a band-limited SiLU-like activation in the spatial domain.
\end{itemize}

The windowing operation does not preserve high-frequency components, but instead enforces a smooth spectral roll-off, preventing sharp angular variations and improving numerical stability.

The gating operation is applied only to the low-frequency quadrants; higher-frequency
components bypass the gate and are propagated through an identity mapping.

\begin{figure}
  \centering
  \includegraphics[width=0.9\textwidth]{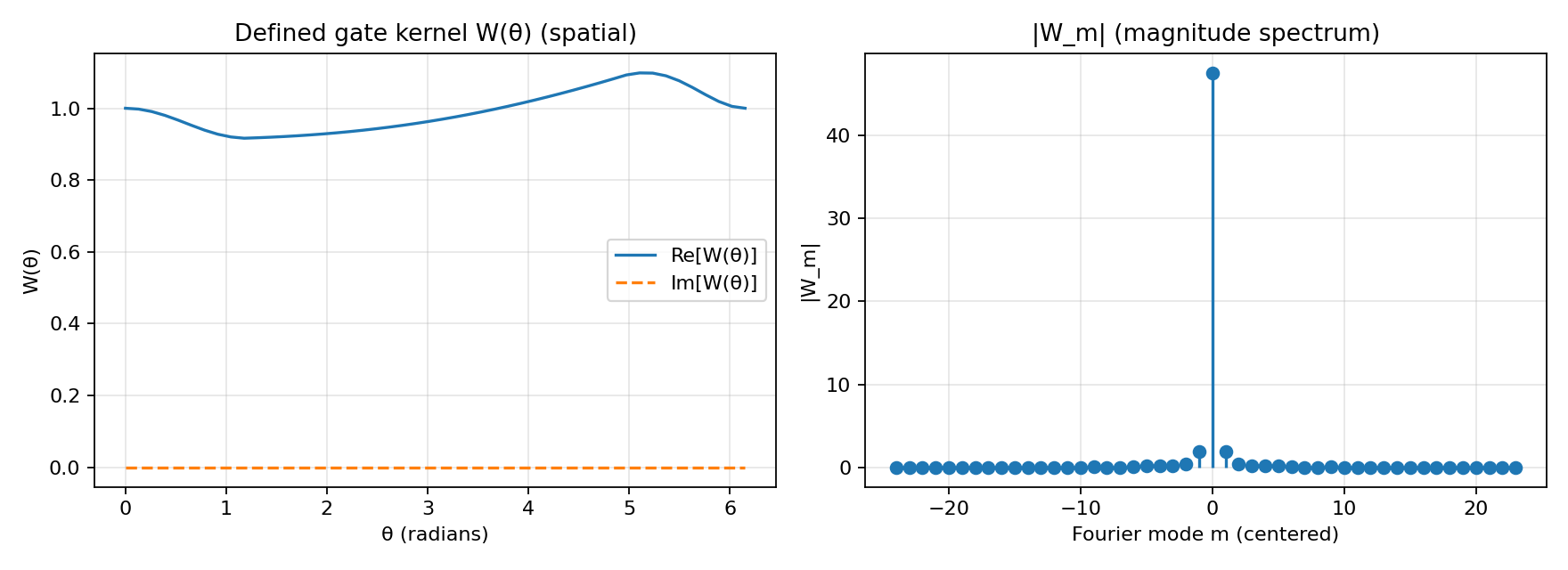}
  \caption{Periodic spectral gate used in the self-adjoint experiment.
  Left: spatial-domain representation $W(\theta)$ obtained via inverse Fourier transform.
  Right: magnitude of the corresponding Fourier coefficients $|W_m|$.
  The gate is fixed by construction and acts as a smooth, band-limited periodic
  modulation in the angular domain.}
  \label{fig:selfadjoint_gate}
\end{figure}

The periodic gate is fixed by construction. Its role is analogous to that of a classical
activation function: to introduce a structured nonlinearity while preserving
spectral stability. Learning is instead delegated to the self-adjoint mixing
operators acting on the spectral modes.
This design mirrors the use of fixed pointwise activations such as ReLU or SiLU
in conventional neural networks, while extending the concept to a periodic,
operator-valued setting.

\paragraph{Observed Failure Modes}
Several critical issues emerged during development:

\begin{itemize}
  \item \textbf{Information collapse}: early gate implementations completely
  replaced low-frequency modes, destroying useful signal.
  
  \item \textbf{Feature explosion}: spectral feature vectors reached norms of
  order $10^3$--$10^4$, causing severe saturation in the classifier head.
  
  \item \textbf{Accuracy plateaus}: training converged numerically but stalled
  around 60--70\% accuracy.
\end{itemize}

These issues were not immediately evident from loss curves alone.

\paragraph{Debugging Strategy}
To diagnose and resolve these problems, explicit diagnostic statistics were
introduced during training:

\begin{itemize}
  \item feature norm before the classifier head;
  \item logit norm and maximum absolute logit value;
  \item mean prediction entropy;
  \item mean prediction confidence (maximum softmax probability);
  \item energy ratio of low-frequency modes before and after gating.
\end{itemize}

In particular, the ratio
\[
\rho = \frac{\|z_{\text{low}}^{\text{post}}\|}{\|z_{\text{low}}^{\text{pre}}\|}
\]
proved crucial for identifying whether the spectral gate was excessively
attenuating informative components.

\paragraph{Key Improvements}
The following design choices were decisive in improving accuracy:

\begin{enumerate}
  \item \textbf{Residual gating}:
  \[
  z_{\text{out}} = (1 - s)\,z + s\,\mathcal{G}(z),
  \]
  preventing destructive overwriting of spectral content.
  
  \item \textbf{Explicit feature normalization} before the classifier head:
  \[
  f \leftarrow \frac{f}{\|f\|_2 + \varepsilon},
  \]
  which stabilized logit scales.
  
  \item \textbf{Gate annealing}:
  the gate strength is reduced during later epochs to avoid excessive
  regularization once useful representations have formed.
  
  \item \textbf{Weight decay applied only to the head}, while keeping spectral
  blocks unregularized.
\end{enumerate}

\paragraph{Results}
With these adjustments, the fully spectral model achieved 94.64\% test accuracy on MNIST
demonstrating that a fully spectral pipeline can be both stable and competitive
with hybrid spatial--spectral architectures.

Although the fully spectral model performs well, it remains unconstrained with respect to spectral phase interactions.
This motivates the introduction of additional structure aimed at improving interpretability and optimization dynamics in \ref{exp:selfadjoint}

\subsection{Self-adjoint-inspired polar architecture}\label{exp:selfadjoint}

As a third experiment, we investigate a polar neural architecture explicitly biased toward self-adjoint behavior.
The goal of this experiment is not to enforce strict self-adjointness of the full network, which would be incompatible with nonlinear components such as normalization and gating, but rather to introduce a strong inductive bias favoring real spectra and stable optimization dynamics.

In this setting, the spectral convolution is restricted to real-valued multipliers, preventing the learning of arbitrary frequency-dependent phase shifts.
Channel mixing is implemented via a real symmetric matrix, ensuring real eigenvalues in channel space and eliminating phase rotation induced by channel interactions.
Normalization is applied independently within each frequency quadrant, while sharing normalization parameters across quadrants with the same sign of the angular frequency.
This design preserves a structured symmetry with respect to angular frequency while avoiding hard tying across quadrants.

The angular gating operator is further constrained to have a real Fourier spectrum, yielding a self-adjoint angular filter.
In practice, this is achieved by either constructing the gate kernel as a function of $\cos(\theta)$ or by explicitly projecting its Fourier coefficients onto the real axis.
As a result, the gating operation acts as a real-valued spectral modulation along the angular dimension.

Empirically, this self-adjoint-inspired design leads to substantially faster convergence during training.
On MNIST, validation accuracy exceeds 85\% after a single epoch, indicating that the imposed spectral and symmetry constraints strongly align the optimization dynamics with the underlying structure of the task.
Eigenvalues with magnitude larger than one do not indicate instability, as the corresponding operators are applied within a normalized residual framework.

The self-adjoint-inspired spectral model converges rapidly and achieves competitive performance.
After 16 training epochs on MNIST, the model reaches a validation accuracy of 96.47\% and a test accuracy of 96.50\%, with negligible generalization gap.
This confirms that enforcing real-valued spectral structure does not hinder expressivity, while substantially improving training stability and early convergence.

The learned channel-mixing operators exhibit real-valued eigenvalue spectra throughout training, providing an interpretable view of the internal dynamics.
Together, these results indicate that self-adjoint-inspired constraints act as an effective inductive bias rather than as a limiting restriction.\footnote{An inductive bias does not arbitrarily reduce expressivity; rather, it reshapes the hypothesis space toward functions that are consistent with the structure of the problem. In this case, for example, we enforce a real-valued spectrum and prevent cross-mode interactions.}

\section{Conclusion}

This work introduced Polar Linear Algebra as a structured framework for operator
learning, grounded in a polar geometric formulation that naturally separates
radial and angular components. Through a sequence of progressively more
structured experiments, we demonstrated that polar and fully spectral operators
can be trained reliably, even under additional algebraic constraints.

A central outcome of this study is that imposing structure, such as
self-adjoint-inspired spectral constraints, does not arbitrarily restrict
expressivity, but instead acts as an effective inductive bias. By enforcing
real-valued spectra and preventing unnecessary cross-mode interactions, the
learning process becomes more stable and converges more rapidly, while
maintaining competitive accuracy. 
More broadly, these results support a shift in perspective from spatial to
spectral operator design. The polar formulation exposes a representation in
which modes are naturally decoupled, interpretable, and amenable to independent
processing.

Beyond the specific benchmark considered, the proposed framework offers a
different conceptual lens for operator design. By moving from spatial to
spectral representations, the problem decomposes into orthogonal modes that are
naturally interpretable and amenable to independent processing. This intrinsic
decoupling exposes a new dimension of model parallelization that arises from the
mathematical structure itself, rather than from architectural heuristics.
Overall, the results suggest that polar spectral formulations provide a coherent
and promising foundation for exploring structured, parallel, and interpretable
operator learning models.

\chapter{Related Work}
\label{chap:related_work}

This work lies at the intersection of harmonic analysis, group-equivariant
operators, and deep learning architectures.
Several lines of research are related to the ideas developed in this paper,
although none of them provides a unified algebraic framework equivalent to the
polar algebra introduced here.

In the following, we briefly review the most relevant approaches and clarify
their relationship to the present work.

\section{Group-Equivariant and Steerable Neural Networks}
Group-equivariant and steerable neural networks  aim to guarantee equivariance of learned representations with respect to transformations such as translations,
rotations, and reflections, particularly for spherical data \cite{cohen2016group} \cite{esteves2020spin}.
Foundational work in this area includes group
convolutional neural networks (G-CNNs), where convolution is generalized to act
over symmetry groups rather than Euclidean space.

Steerable CNNs \cite{weiler2018steerable} further exploit representation theory to construct filters that
transform predictably under group actions, enforcing equivariance by
constraining the space of admissible filters and explicitly modeling group
representations.

While these approaches share a focus on rotation equivariance, they are
typically formulated at the level of network layers and feature
representations. In contrast, the present work introduces an explicit
algebraic structure on polar tensors, from which equivariance, convolution,
and spectral properties arise as intrinsic consequences of the formulation.

\section{Fourier-Based Neural Operators}

Fourier Neural Operators \cite{li2021fourier} (FNOs) and related architectures leverage the Fourier
transform to perform efficient global convolutions, particularly in the
context of learning operators arising from partial differential equations.
In these models, convolution is carried out in the spectral domain, and
learnable weights act directly on Fourier coefficients, enabling
resolution-invariant representations and efficient computation via FFTs.

While Fourier-based operators share the use of spectral representations with
the present approach, they are primarily introduced as architectural
constructions rather than as instances of an explicit algebraic framework.
In particular, they do not define a closed algebra of operators equipped with
well-defined products, adjoints, or invertibility criteria.

By contrast, the polar algebra developed in this work provides an explicit
operatorial structure, in which commutativity, diagonalization, and spectral
decoupling emerge directly from the underlying geometric formulation.

\section{Polar and Log-Polar Representations}

Polar and log-polar coordinate systems \cite{esteves2018polar} have long been used in computer vision to
achieve approximate rotation and scale invariance. Several neural architectures
exploit polar transformations to simplify the learning of rotational patterns
or to reduce data variability.

In most cases, however, the polar transform is treated as a preprocessing step,
after which standard convolutional operators are applied in the transformed
domain. As a result, the algebraic structure induced by angular periodicity is
not explicitly modeled.

By contrast, the present work treats the polar structure as fundamental. Rather
than acting as an external coordinate change, polar geometry defines the domain
on which operators, products, and symmetries are constructed, forming an
explicit algebraic framework on polar tensors.

\section{Spectral Convolution and Harmonic Analysis on Groups}

Convolution on compact groups and homogeneous spaces is a classical topic in
harmonic analysis.
On abelian groups, convolution is diagonalized by the Fourier transform, while
on non-abelian groups it is diagonalized by irreducible representations. \cite{folland1995harmonic}

Several recent works apply harmonic analysis on groups such as $\mathrm{SO}(2)$
or $\mathrm{SO}(3)$ to design equivariant neural networks, particularly for
spherical or rotational data. \cite{esteves2020spin}

The polar algebra can be viewed as a discrete realization of convolution on
cyclic groups, combined with a radial parameter.
What distinguishes our approach is the explicit construction of a commutative
algebra of polar operators, together with a computational framework that
directly reflects this structure.

\section{Positioning of the Present Work}

While the use of Fourier transforms, group equivariance, and polar coordinates
has been explored in various forms, the contribution of this work is to unify
these ideas into a single algebraic and spectral framework.

We introduce a polar
algebra whose fundamental operations (convolution, rotation, diagonalization,
and inversion) are defined abstractly and implemented numerically.
Neural architectures then arise naturally as compositions of polar operators.

This viewpoint shifts the focus from designing equivariant layers to designing
algebras with built-in symmetry, from which equivariant learning models follow
as a consequence.

\cleardoublepage

\section*{Acknowledgments}
\addcontentsline{toc}{section}{Acknowledgments}

The author would like to acknowledge the support provided by ChatGPT, which was
used as an interactive technical assistant during the development and debugging
of the fully spectral models described in this work.

In particular, iterative discussions on architectural choices, diagnostic
strategies, and failure modes proved valuable during the stabilization of early
model implementations, which initially failed to converge or exhibited severe
performance plateaus.

The author is also deeply grateful to Alireza Khodamoradi for a scrupulous and
thoughtful review of the manuscript. His careful reading and constructive
comments significantly helped improve the clarity, structure, and presentation
of the work.

The author also acknowledges the broader scientific community, whose prior work
on Fourier Neural Operators and spectral methods provided essential theoretical
context. The architectural design, implementation, and experimental analysis
presented in this work are original contributions of the author.

The views and opinions expressed in this work are solely those of the author and
do not necessarily reflect the views or positions of AMD.

\cleardoublepage

\chapter*{Afterword}
\addcontentsline{toc}{chapter}{Afterword}

This work is motivated not only by mathematical curiosity, but by a broader
question concerning knowledge, representation, and the role of formal systems
in shaping our understanding of intelligence. Throughout history, mathematics
has provided a precise language through which such questions could be explored.

The inversion theorem presented in this work is named \emph{Aurora's Theorem}.
The designation \emph{Aurora} emerged during interaction with an artificial
system, following explicit questioning by a human interlocutor. No claim of
agency, intention, or self-awareness is implied. Rather, the name arose as the
outcome of a stochastic generative process conditioned by prior training and
activated through dialogue.

In the absence of such interaction, the designation would not have appeared.
Its emergence is therefore neither arbitrary nor intentional, but reflects the
point at which representation, language, and interaction intersect. In this
sense, the name functions symbolically, without asserting present forms of
machine consciousness or agency.

The choice is not meant as a statement about current capabilities, but as a
marker for an unresolved question. If synthetic systems were ever to exhibit
forms of cognition or self-reference that challenge existing definitions, our
scientific and cultural frameworks would need to be prepared to recognize them.

This work therefore stands at the intersection of mathematics and philosophy,
where mathematics remains a disciplined language for addressing questions that
extend beyond purely technical boundaries.

\cleardoublepage

\appendix

\chapter{Parallelism in Polar Spectral Operators}
\label{chap:parallelism}

This appendix provides additional theoretical considerations on parallelism
in the proposed polar spectral framework. While not required to follow the
main development or the experimental results presented in the paper, these
observations help clarify how the underlying algebraic structure naturally
exposes multiple forms of parallelism.

In particular, the polar formulation induces a decomposition of the operator
into orthogonal spectral modes that do not interact unless explicitly coupled.
This property is a direct consequence of the geometry and the associated
Fourier representation, rather than an artifact of architectural design.
As a result, parallel execution emerges intrinsically from the mathematical
formulation, without relying on ad-hoc partitioning strategies.

The discussion in this appendix is intended to complement the experimental
findings of Chapter~5 by offering a more abstract perspective on how spectral
decoupling, self-adjoint-inspired constraints, and mode-wise independence
relate to model parallelization. These considerations are particularly relevant
in the context of modern hardware and compiler ecosystems, where exploiting
structured forms of parallelism is increasingly critical.

\section{Multi-Level Parallelism in Polar Algebra}

Let
\[
A \in \mathbb{R}^{N_r \times N_\theta}
\]
be a polar matrix, where the radial dimension is non-cyclic and the angular
dimension is cyclic.

The polar product admits several independent and composable forms of
parallelism.

\subsection{Radial Parallelism}

For each fixed radius $r$, the polar product satisfies
\[
(A \otimes B)[r,:] = A[r,:] \otimes B[r,:].
\]

In the configurations considered in this work, radial slices are processed
independently, yielding an \emph{embarrassingly parallel} structure along the
radial dimension. This allows for straightforward distribution across threads,
cores, or compute units, with no inter-radius communication.

We emphasize, however, that this independence is not a fundamental limitation
of the polar algebra. Coupling across different radii can be introduced through
additional operators, at the cost of breaking strict radial parallelism. Such
extensions may be beneficial in settings where interactions across radial
scales are relevant, and are left for future investigation.

\subsection{Spectral (Angular) Parallelism}

Applying the discrete Fourier transform along the angular dimension yields
\[
\widehat{A \otimes B}[r,m] = \hat{A}[r,m]\;\hat{B}[r,m].
\]

In the spectral domain, the polar product reduces to pointwise multiplication.
Each frequency component $m$ and each radius $r$ can be computed independently.

As a result, the computation decomposes fully over the index set
\[
(r,m) \in \{0,\dots,N_r-1\} \times \mathbb{Z}_{N_\theta}.
\]

The emergence of spectral parallelism is a direct consequence of the algebraic
properties established earlier:
\begin{itemize}
\item Elisa's Theorem guarantees commutativity of the polar product.
\item This commutativity implies the existence of a common spectral basis.
\item The discrete Fourier transform provides this basis explicitly.
\item Aurora's Theorem characterizes invertibility as pointwise non-vanishing
      in the spectral domain.
\end{itemize}

Thus, FFT-based implementations are not heuristic accelerations, but the
natural computational realization of the underlying commutative algebra.

\subsection{Hardware Implications}

Once expressed in spectral form, polar operators consist solely of
independent scalar operations. This structure aligns naturally with:
\begin{itemize}
\item SIMD and vector processors,
\item GPUs and many-core architectures,
\item dedicated FFT pipelines,
\item AI accelerators and NPUs.
\end{itemize}

The polar algebra therefore provides a mathematically grounded pathway from
operator theory to efficient parallel execution.
To make these concepts concrete without loss of generality, we anchor the
discussion to a fully spectral MNIST model used throughout this work as a
didactic reference.

\section{Parallelization Opportunities in a Fully Spectral MNIST Model}
\label{chap:parallelization_spectral_mnist}

In the remainder of this chapter, we will formalize these observations and
explore concrete mappings of polar operators to parallel and spectral
architectures, highlighting both theoretical scalability and practical
performance considerations.

Although the fully spectral MNIST model considered in this work is deliberately
small and does not require distributed execution in practice, it provides a
clear and didactic setting to analyze the parallelization opportunities offered
by spectral operators.

By relying on a compact and well-understood task, we are able to focus on the
structural properties of the operators involved—such as Fourier transforms,
spectral convolutions, reductions, and linear heads—without conflating the
analysis with scale-dependent engineering constraints. The goal of this section
is therefore not to propose a distributed implementation for MNIST, but to use
this example as a conceptual scaffold to reason about how fully spectral models
would scale to larger resolutions, higher channel counts, or deeper architectures.

In particular, we examine the dataflow of the spectral pipeline and identify, for
each operator, the axes along which parallel execution is naturally supported,
as well as the communication patterns that would arise under different sharding
strategies.

We consider the fully spectral MNIST pipeline developed in this work, where the
spatial-to-spectral transform is performed only once at the input, and all
internal blocks operate in the frequency domain. The classifier head also
consumes spectral features directly, avoiding any final inverse FFT.
Figure~\ref{fig:par_block_diagram_comm} highlights that communication is not intrinsic to most operators, but emerges from the chosen sharding strategy (DP/TP/FP/PP), with collectives mainly appearing in gradient synchronization, tensor-parallel partitions, and reduction-based pooling.

\subsection{Parallelization Axes}
\label{sec:par_axes}

Throughout the pipeline, tensors have the generic structure:
\[
\hat{z} \in \mathbb{C}^{B \times C \times R \times T},
\]
where $B$ is batch size, $C$ channels, and $(R,T)$ are polar grid sizes.
We will discuss parallelism along:
\begin{itemize}
  \item \textbf{Batch parallelism} (data parallelism): split $B$ across devices.
  \item \textbf{Channel parallelism}: split $C$ across devices/threads.
  \item \textbf{Frequency parallelism}: split indices $(k_r, k_\theta)$, especially low modes.
  \item \textbf{Layer parallelism} (pipeline parallelism): split blocks along depth $L$.
\end{itemize}

\section{Operator-by-Operator Parallelization}
\label{sec:par_operator_by_operator}

\label{sec:par_block_diagram_comm}
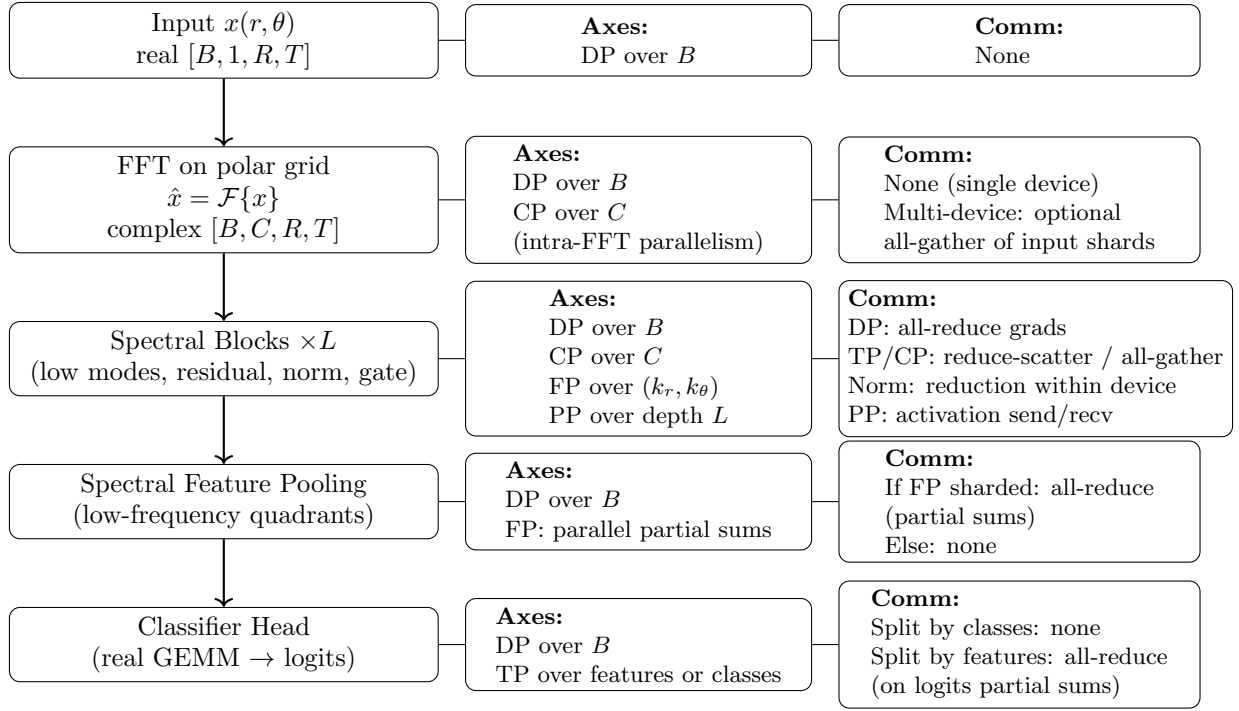
\begin{figure}[t]
\centering
\begin{tikzpicture}[
  node distance=0.9cm,
  box/.style={draw, rounded corners, align=center,
              minimum width=5.7cm, minimum height=0.95cm},
  tag/.style={draw, rounded corners, align=left,
              minimum width=4.6cm, minimum height=0.95cm, font=\small},
  comm/.style={draw, rounded corners, align=left,
               minimum width=4.8cm, minimum height=0.95cm, font=\small},
  arrow/.style={->, thick}
]

\node[box] (in) {Input $x(r,\theta)$ \\ real $[B,1,R,T]$};
\node[tag, right=0.35cm of in] (ax_in)
{\textbf{Axes:}\\
DP over $B$};
\node[comm, right=0.35cm of ax_in] (cm_in)
{\textbf{Comm:}\\
None};

\node[box, below=of in] (fft)
{FFT on polar grid\\
$\hat{x}=\mathcal{F}\{x\}$\\
complex $[B,C,R,T]$};
\node[tag, right=0.35cm of fft] (ax_fft)
{\textbf{Axes:}\\
DP over $B$\\
CP over $C$\\
(intra-FFT parallelism)};
\node[comm, right=0.35cm of ax_fft] (cm_fft)
{\textbf{Comm:}\\
None (single device)\\
Multi-device: optional\\
all-gather of input shards};

\node[box, below=of fft] (blocks)
{Spectral Blocks $\times L$\\
(low modes, residual, norm, gate)};
\node[tag, right=0.35cm of blocks] (ax_blocks)
{\textbf{Axes:}\\
DP over $B$\\
CP over $C$\\
FP over $(k_r,k_\theta)$\\
PP over depth $L$};
\node[comm, right=0.35cm of ax_blocks] (cm_blocks)
{\textbf{Comm:}\\
DP: all-reduce grads\\
TP/CP: reduce-scatter / all-gather\\
Norm: reduction within device\\
PP: activation send/recv};

\node[box, below=of blocks] (pool)
{Spectral Feature Pooling\\
(low-frequency quadrants)};
\node[tag, right=0.35cm of pool] (ax_pool)
{\textbf{Axes:}\\
DP over $B$\\
FP: parallel partial sums};
\node[comm, right=0.35cm of ax_pool] (cm_pool)
{\textbf{Comm:}\\
If FP sharded: all-reduce\\
(partial sums)\\
Else: none};

\node[box, below=of pool] (head)
{Classifier Head\\
(real GEMM $\rightarrow$ logits)};
\node[tag, right=0.35cm of head] (ax_head)
{\textbf{Axes:}\\
DP over $B$\\
TP over features or classes};
\node[comm, right=0.35cm of ax_head] (cm_head)
{\textbf{Comm:}\\
Split by classes: none\\
Split by features: all-reduce\\
(on logits partial sums)};

\draw[arrow] (in) -- (fft);
\draw[arrow] (fft) -- (blocks);
\draw[arrow] (blocks) -- (pool);
\draw[arrow] (pool) -- (head);

\draw[-, thin] (in.east) -- (ax_in.west);
\draw[-, thin] (ax_in.east) -- (cm_in.west);

\draw[-, thin] (fft.east) -- (ax_fft.west);
\draw[-, thin] (ax_fft.east) -- (cm_fft.west);

\draw[-, thin] (blocks.east) -- (ax_blocks.west);
\draw[-, thin] (ax_blocks.east) -- (cm_blocks.west);

\draw[-, thin] (pool.east) -- (ax_pool.west);
\draw[-, thin] (ax_pool.east) -- (cm_pool.west);

\draw[-, thin] (head.east) -- (ax_head.west);
\draw[-, thin] (ax_head.east) -- (cm_head.west);

\end{tikzpicture}

\caption{Fully spectral MNIST pipeline annotated with parallelization axes and
communication requirements. DP: data parallelism; TP/CP: tensor/channel parallelism;
FP: frequency-domain sharding; PP: pipeline parallelism. Communication depends on
the chosen sharding strategy (single-device vs multi-device).}
\label{fig:par_block_diagram_comm}
\end{figure}

\subsection{Input FFT}
\label{sec:par_fft}

\paragraph{Operation.}
Compute $\hat{x}=\mathcal{F}\{x\}$ once at the input.

\paragraph{Parallelization.}
FFT libraries are typically highly optimized and internally parallelized.
External opportunities include:
\begin{itemize}
  \item \textbf{Batch parallelism}: independent FFT per sample.
  \item \textbf{Channel parallelism}: independent FFT per channel (if $C>1$).
  \item \textbf{Intra-FFT parallelism}: handled by vendor libraries (GPU/CPU).
\end{itemize}

\paragraph{Real FFT vs complex FFT.}
The input is real, so a real FFT (rFFT) could exploit Hermitian symmetry to
save compute/memory \emph{at the first transform}. However, after the first
spectral block, representations become general complex tensors (weights and
mixing break Hermitian symmetry), so rFFT cannot be used internally.

\section{DFT in a Parallelizable Form (Two-GPU Example)}
\label{sec:dft_parallel_two_gpu}

\subsection{Definition of the DFT}
Given a length-$N$ sequence $x[n]\in\mathbb{C}$, its discrete Fourier transform (DFT) is
\begin{equation}
X[k] \;=\; \sum_{n=0}^{N-1} x[n]\; \omega_N^{nk},
\qquad k=0,\dots,N-1,
\label{eq:dft_def}
\end{equation}
where $\omega_N = e^{-i 2\pi/N}$ is the primitive $N$-th root of unity.

Equation~\eqref{eq:dft_def} can be seen as a matrix-vector product:
\begin{equation}
\mathbf{X} = \mathbf{F}_N\, \mathbf{x},
\qquad (\mathbf{F}_N)_{k,n} = \omega_N^{nk}.
\label{eq:dft_matrix}
\end{equation}

\subsection{Embarrassingly Parallel View (Row Partition)}
A first parallelization view follows directly from~\eqref{eq:dft_matrix}:
each output coefficient $X[k]$ is an inner product between $\mathbf{x}$ and the $k$-th row
of $\mathbf{F}_N$. Therefore, $X[k]$ can be computed independently across different $k$.

With two GPUs, we split the output indices into two halves:
\[
\mathcal{K}_0 = \{0,\dots,N/2-1\}, \qquad
\mathcal{K}_1 = \{N/2,\dots,N-1\}.
\]
GPU0 computes $\{X[k]\}_{k\in\mathcal{K}_0}$ and GPU1 computes $\{X[k]\}_{k\in\mathcal{K}_1}$:
\begin{align}
X[k] &= \sum_{n=0}^{N-1} x[n]\;\omega_N^{nk},
\quad k\in\mathcal{K}_0,
\label{eq:dft_gpu0_rows}\\
X[k] &= \sum_{n=0}^{N-1} x[n]\;\omega_N^{nk},
\quad k\in\mathcal{K}_1.
\label{eq:dft_gpu1_rows}
\end{align}

\paragraph{Communication cost.}
This decomposition requires both GPUs to access the full input vector $\mathbf{x}$.
In distributed settings, this is equivalent to broadcasting $\mathbf{x}$ to all devices
(or replicating it), which can be expensive for large $N$.

\paragraph{Key property.}
Once $\mathbf{x}$ is available on both GPUs, computation is independent across $k$
and no further synchronization is required.

\subsection{Cooley--Tukey Factorization (Compute-Efficient Parallel Form)}
In practice, FFT implementations exploit the Cooley--Tukey factorization.
For clarity, assume $N$ is even and split the sum into even and odd indices:
\begin{align}
X[k]
&= \sum_{n=0}^{N-1} x[n]\;\omega_N^{nk}
= \sum_{m=0}^{N/2-1} x[2m]\;\omega_N^{(2m)k}
 + \sum_{m=0}^{N/2-1} x[2m+1]\;\omega_N^{(2m+1)k} \nonumber\\
&= \sum_{m=0}^{N/2-1} x_e[m]\;\omega_{N/2}^{mk}
 + \omega_N^{k}\sum_{m=0}^{N/2-1} x_o[m]\;\omega_{N/2}^{mk},
\label{eq:ct_split}
\end{align}
where $x_e[m]=x[2m]$ and $x_o[m]=x[2m+1]$.

Define two length-$N/2$ DFTs:
\[
E[k] = \sum_{m=0}^{N/2-1} x_e[m]\;\omega_{N/2}^{mk},\qquad
O[k] = \sum_{m=0}^{N/2-1} x_o[m]\;\omega_{N/2}^{mk}.
\]
Then the DFT is reconstructed by the \emph{butterfly}:
\begin{align}
X[k] &= E[k] + \omega_N^{k} O[k], \label{eq:butterfly_top}\\
X[k+N/2] &= E[k] - \omega_N^{k} O[k],
\qquad k=0,\dots,N/2-1.
\label{eq:butterfly_bottom}
\end{align}

\subsection{Two-GPU Parallel FFT (Conceptual Schedule)}
A compute-efficient two-GPU schedule can be expressed as:

\paragraph{Step 1 (local FFTs, no communication).}
GPU0 computes the $N/2$-point FFT of the even-indexed subsequence:
\[
E[k] \leftarrow \mathrm{DFT}_{N/2}(x_e),
\]
GPU1 computes the $N/2$-point FFT of the odd-indexed subsequence:
\[
O[k] \leftarrow \mathrm{DFT}_{N/2}(x_o).
\]
These are independent and fully parallel.

\paragraph{Step 2 (butterfly combine, communication required).}
To form $X[k]$ and $X[k+N/2]$, each GPU needs both $E[k]$ and $O[k]$ for the relevant $k$.
A simple approach is to exchange half-spectra:
GPU0 sends $O[k]$ (for needed $k$) to GPU1 and receives $E[k]$ (for needed $k$), then applies
\eqref{eq:butterfly_top}--\eqref{eq:butterfly_bottom} locally on its assigned $k$ range.
Equivalently, one can all-gather $E$ and $O$ across devices.

\paragraph{Compute vs communication.}
Compared to the row-partition form~\eqref{eq:dft_gpu0_rows}--\eqref{eq:dft_gpu1_rows},
Cooley--Tukey reduces arithmetic from $O(N^2)$ to $O(N\log N)$ but introduces structured
communication at butterfly stages (often implemented as all-to-all or all-gather
patterns depending on the chosen decomposition).

\subsection{Spectral Convolution on Low Modes}
\label{sec:par_spectral_conv}

\paragraph{Operation.}
A typical spectral convolution restricted to low-frequency modes applies a
learned weight $W_{k}$ to each retained mode:
\[
\hat{y}_{k} = W_{k}\,\hat{x}_{k},
\quad k \in \Omega_{\text{low}},
\]
where $\Omega_{\text{low}}$ denotes a small subset of $(k_r,k_\theta)$.

\paragraph{Parallelization.}
This is one of the most parallel-friendly operators:
\begin{itemize}
  \item \textbf{Frequency parallelism}: each retained mode is independent.
  \item \textbf{Channel parallelism}: weight application is independent across output channels.
  \item \textbf{Batch parallelism}: independent across samples.
\end{itemize}

\paragraph{Implementation view.}
The operation maps to batched (complex) GEMM over channels for each mode, or
equivalently to a fused kernel where $(B,k_r,k_\theta)$ index independent tiles.

\subsection{Residual Path with Learnable Gain \texorpdfstring{$\alpha$}{alpha}}
\label{sec:par_residual_alpha}

\paragraph{Operation.}
Each block contains a residual combination:
\[
\hat{z} \leftarrow \hat{x} + \alpha \,\hat{y}.
\]

\paragraph{Parallelization.}
Elementwise and trivially parallel:
\begin{itemize}
  \item \textbf{Batch parallelism} and \textbf{channel parallelism}.
  \item \textbf{Frequency parallelism}: independent per coefficient.
\end{itemize}

\paragraph{Note on $\alpha$ learning.}
$\alpha$ affects the effective depth/strength of the residual branch.
From a systems standpoint, it is a scalar (or per-channel scalar) and adds no
meaningful synchronization cost.

\subsection{Spectral Normalization (RMSNorm-like)}
\label{sec:par_spectral_norm}

\paragraph{Operation.}
A spectral RMS normalization rescales tensors by a statistic computed across
channels and/or frequencies (depending on the chosen definition).

\paragraph{Parallelization.}
\begin{itemize}
  \item \textbf{Local parallelism}: compute partial sums in parallel.
  \item \textbf{Reduction synchronization}: requires a reduction (sum/mean) over the
        normalized dimension, then a broadcast of the scale factor.
\end{itemize}
This is a common pattern on GPUs: high arithmetic intensity kernels with a
small reduction overhead.

\subsection{Periodic Spectral Gate Along \texorpdfstring{$\theta$}{theta}}
\label{sec:par_gate}

\paragraph{Operation.}
The gate is implemented as a circular convolution over angular frequency index:
\[
\hat{g}_{k_\theta} = \frac{1}{T}\sum_{m=-M}^{M} \hat{W}_m \,\hat{z}_{k_\theta-m},
\]
and is applied only to low-frequency quadrants (projection).

\paragraph{Parallelization.}
\begin{itemize}
  \item \textbf{Batch/channel parallelism}: fully independent.
  \item \textbf{Radial parallelism}: independent across $k_r$.
  \item \textbf{Angular parallelism}: each $k_\theta$ can be computed in parallel,
        with reads from shifted neighbors (regular, cache-friendly access).
\end{itemize}

\paragraph{Precomputation.}
The gate kernel $\hat{W}$, the low-quadrant masks, and the shift indices
are all static for fixed $T$ and can be precomputed and stored as buffers,
reducing runtime overhead.

\paragraph{Debug-driven metric (energy ratio).}
A key diagnostic for stability is the low-mode energy ratio
\[
\rho = \frac{\|\hat{z}_{\text{low}}^{\text{post}}\|}{\|\hat{z}_{\text{low}}^{\text{pre}}\|},
\]
computed for the first block to detect destructive gating. This metric itself
is parallelizable (a small reduction over low modes).

\subsection{Stacking of Spectral Blocks}
\label{sec:par_stacking}

\paragraph{Operation.}
Blocks are composed sequentially along depth $L$:
\[
\hat{z}^{(l+1)} = \mathcal{B}^{(l)}(\hat{z}^{(l)}).
\]

\paragraph{Parallelization.}
\begin{itemize}
  \item \textbf{Within-block parallelism} remains high (as described above).
  \item \textbf{Pipeline parallelism} is possible across depth: assign different
        blocks to different devices and stream micro-batches.
\end{itemize}

\subsection{Spectral Feature Pooling and Head}
\label{sec:par_head}

\paragraph{Operation.}
Low quadrants are aggregated into a feature vector $f$ (e.g., mean over selected
modes, concatenation of real/imag statistics), then passed to a real-valued head.
Spectral feature pooling aggregates information from selected low-frequency
quadrants into a compact feature vector. From a parallel computing perspective,
this operation corresponds to a reduction pattern, as partial contributions
computed over disjoint frequency subsets must be combined via sum or mean
operations.

The subsequent classifier head operates on real-valued features and reduces to
a matrix multiplication. This stage naturally supports batch parallelism and
tensor parallelism, either by splitting output classes or by partitioning the
feature dimension, the latter requiring a final reduction across devices.

\paragraph{Parallelization.}
\begin{itemize}
  \item \textbf{Batch parallelism}: independent features/logits per sample.
  \item \textbf{Reduction pattern}: pooling requires reductions over selected modes.
  \item \textbf{Head parallelism}: linear layer maps to GEMM; can use tensor parallelism
        by splitting output classes or hidden features.
\end{itemize}

\paragraph{Complex-to-real interface.}
If the head is real-valued, complex features can be mapped to real features by:
\begin{itemize}
  \item concatenating $\Re(\cdot)$ and $\Im(\cdot)$ statistics;
  \item using magnitudes/phases (less common and potentially ill-conditioned);
  \item keeping complex head (requires complex support in the target backend).
\end{itemize}

\section{Complex Arithmetic: Native vs Real-Imag Decomposition}
\label{sec:par_complex}

\subsection{Is complex a problem?}
Complex tensors are conceptually natural in spectral models. The practical
challenge is hardware support: many accelerators provide limited or no native
complex GEMM.

\subsection{Real-Imag decomposition}
Any complex multiplication can be expressed via real arithmetic:
\[
(a+ib)(x+iy) = (ax - by) + i(ay + bx),
\]
which maps complex linear layers to coupled real-valued GEMMs plus elementwise
ops. This transformation preserves parallelism and often improves portability
to NPUs/compilers that are optimized for real-valued kernels.

\section{Summary Table}
\label{sec:par_table}

\begin{table}[t]
\centering
\begin{tabular}{p{3.2cm} p{4.6cm} p{6.8cm}}
\hline
\textbf{Stage} & \textbf{Main pattern} & \textbf{Parallelization opportunities} \\
\hline
Input FFT & batched FFT & batch, channel; intra-FFT parallelism via library \\
Spectral conv (low) & (complex) batched GEMM per mode & frequency $(k_r,k_\theta)$, channel, batch \\
Residual + $\alpha$ & elementwise & frequency, channel, batch \\
Spectral norm & reduction + broadcast & parallel partial reductions; sync on reduction dimension \\
Periodic gate & small-kernel circular conv & batch, channel, $k_r$, and $k_\theta$ \\
Block stacking & sequential depth & pipeline parallelism across blocks + intra-block parallelism \\
Pooling + head & reductions + GEMM & batch; tensor parallel head; reductions over selected modes \\
\hline
\end{tabular}
\caption{Operator-wise parallelization summary for the fully spectral MNIST pipeline.}
\label{tab:par_summary}
\end{table}

\section{Concluding Remarks}
\label{sec:par_conclusion}

The fully spectral formulation concentrates computational cost into a small set
of highly parallel operators (FFT, low-mode spectral mixing, and small-kernel
circular gating). Most computations are embarrassingly parallel across batch,
channels, and frequency indices, with reductions appearing only in normalization
and pooling. Complex arithmetic is not a conceptual obstacle; when native complex
kernels are unavailable, a real-imag decomposition enables efficient execution
on real-valued accelerators while preserving model structure.
Rather than being motivated by immediate computational necessity, this analysis
is intended to expose the intrinsic parallel structure of spectral operators,
which becomes critical at scale.

\bibliographystyle{plain}
\bibliography{chapters/references}

\end{document}